\documentclass{article}

\usepackage{amsmath,amsthm,amssymb,amsfonts,mathrsfs}
\usepackage{bbm}
\usepackage{tensor}
\usepackage{hyperref}
\usepackage{fullpage}
\usepackage{graphicx}
\usepackage{enumerate}

\usepackage{algorithm}
\usepackage{algpseudocode}

\newcommand{\Rbb}{\mathbb{R}}

\newcommand{\Nbb}{\mathbb{N}}

\newcommand{\Dcal}{\mathcal{D}}
\newcommand{\Ecal}{\mathcal{E}}
\newcommand{\Fcal}{\mathcal{F}}
\newcommand{\Gcal}{\mathcal{G}}
\newcommand{\Hcal}{\mathcal{H}}
\newcommand{\Ical}{\mathcal{I}}

\newcommand{\Lcal}{\mathcal{L}}
\newcommand{\Mcal}{\mathcal{M}}
\newcommand{\Ncal}{\mathcal{N}}

\newcommand{\Pcal}{\mathcal{P}}

\newcommand{\Rcal}{\mathcal{R}}

\newcommand{\Tcal}{\mathcal{T}}

\newcommand{\Vcal}{\mathcal{V}}

\newcommand{\Xcal}{\mathcal{X}}

\newcommand{\Lscr}{\mathscr{L}}

\newcommand \Tau {T}

\DeclareMathOperator*{\var}{var}
\DeclareMathOperator*{\Exp}{\mathbb{E}}

\DeclareMathOperator*{\Prob}{\mathbb{P}}

\newcommand{\indicator}{\mathbbm{1}}
\newcommand\ind[1]{\indicator_{\parenth{#1}}}

\newcommand \func[5]{
\[
\begin{aligned}
#1\colon #2 &\rightarrow #3 \\
#4 &\mapsto #5
\end{aligned}
\]
}

\newcommand \inlinefunc[5]{$#1\colon #4 \in #2 \mapsto #5$}

\newcommand \al[1]{\begin{align*}
#1
\end{align*}
}

\newcommand \mat[1]{
\left(
\begin{array}
#1
\end{array}
\right)
}

\newcommand \bra{\left\langle}
\newcommand \ket{\right\rangle}
\newcommand \braket[2]{\bra #1, #2 \ket}
\newcommand \parenth[1]{\left( #1 \right)}
\newcommand \sqbr[1]{\left[ #1 \right]}

\newtheorem{theorem}{Theorem}
\newtheorem{lemma}{Lemma}
\newtheorem{corollary}{Corollary}

\newtheorem{rem}{Remark}
\newtheorem{fact}{Fact}
\newtheorem{assumption}{Assumption}
\newtheorem{definition}{Definition}
\newtheorem{proposition}{Proposition}

\usepackage[usenames,dvipsnames]{xcolor}

\makeatletter
\newcommand{\bnjcomment}[1]{%
\@ifundefined{showcomments}{\relax}{%
{\color{Red} #1}%
}%
}
\makeatother

\makeatletter
\newcommand{\walidcomment}[1]{%
\@ifundefined{showcomments}{\relax}{%
{\color{Red} Walid: #1}%
}%
}
\makeatother

\newcommand{\Simp}{\Delta}

\newcommand{\mucont}{\boldsymbol\mu}

\newcommand{\cesaro}{Ces\`{a}ro}
\newcommand{\toCes}{\xrightarrow{(\gamma_\tau)}}

\title{Learning Nash Equilibria in Congestion Games
\thanks{This work was supported in part by FORCES (Foundations Of Resilient CybEr-physical Systems), which receives support from the National Science Foundation (NSF award numbers CNS-1238959, CNS-1238962, CNS-1239054, CNS-1239166).}}

\author{Walid Krichene\thanks{Walid Krichene is with the department of Electrical Engineering and Computer Sciences, UC Berkeley ({\tt walid@eecs.berkeley.edu}).}
\and Benjamin Drigh\`{e}s\thanks{Benjamin Drigh\`{e}s is with the Ecole Polytechnique, Palaiseau, France ({\tt benjamin.drighes@polytechnique.edu}).}
\and Alexandre M. Bayen\thanks{Alexandre M. Bayen is with the department of Electrical Engineering and Computer Sciences and the department of Civil and Environmental Engineering, UC Berkeley ({\tt bayen@berkeley.edu}).}}
\date{}

\begin{document}

\maketitle

\begin{abstract}
We study the repeated congestion game, in which multiple populations of players share resources, and make, at each iteration, a decentralized decision on which resources to utilize. We investigate the following question: given a model of how individual players update their strategies, does the resulting dynamics of strategy profiles converge to the set of Nash equilibria of the one-shot game? We consider in particular a model in which players update their strategies using algorithms with sublinear discounted regret. We show that the resulting sequence of strategy profiles converges to the set of Nash equilibria in the sense of Ces\`{a}ro means. However, strong convergence is not guaranteed in general. We show that strong convergence can be guaranteed for a class of algorithms with a vanishing upper bound on discounted regret, and which satisfy an additional condition. We call such algorithms AREP algorithms, for Approximate REPlicator, as they can be interpreted as a discrete-time approximation of the replicator equation, which models the continuous-time evolution of population strategies, and which is known to converge for the class of congestion games. In particular, we show that the discounted Hedge algorithm belongs to the AREP class, which guarantees its strong convergence. 
\end{abstract}

\thispagestyle{plain}

\section{Introduction}

Congestion games are non-cooperative games that model the interaction of players who share resources.  Each player makes a decision on which resources to utilize. The individual decisions of players result in a resource allocation at the population scale. Resources which are highly utilized become congested, and the corresponding players incur higher losses. For example, in routing games --a sub-class of congestion games, the resources are edges in a network, and each player needs to travel from a given source vertex to a given destination vertex on the graph. Each player chooses a path, and the joint decision of all players determines the congestion on each edge. The more a given edge is utilized, the more congested it is, creating delays for those players using that edge.

The one-shot congestion game has been studied extensively, and a comprehensive presentation is given for example in~\cite{roughgarden2007}. In particular, congestion games are shown to be potential games, thus their Nash equilibria can be expressed as the solution to a convex optimization problem. Characterizing the Nash equilibria of the congestion game gives useful insights, such as the loss of efficiency due to selfishness of players. One popular measure of inefficiency is the price of anarchy, introduced by Koutsoupias and Papadimitriou in \cite{koutsoupias1999worst}, and studied in the case of routing games by Roughgarden et al. in \cite{roughgarden2002bad}. While characterizing Nash equilibria of the one-shot congestion game gives many insights, it does not model how players \emph{arrive to the equilibrium}. Studying the game in a repeated setting can help answer this question. Additionally, most realistic scenarios do not correspond to a one-shot setting, but rather a repeated setting in which players make decisions in an online fashion, observe outcomes, and may update their strategies given the previous outcomes. 
This motivates the study of the game and the population dynamics in an online learning framework.

Arguably, a good model for learning should be distributed, and should not have extensive information requirements. In particular, one should not expect the players to have an accurate model of congestion of the different resources. Players should be able to learn simply by observing the outcomes of their previous actions, and those of other players. No-regret learning is of particular interest here, as many regret-minimizing algorithms are easy to implement by individual players, and only require the player losses to be revealed. The Hedge algorithm (also known as boosting or the exponential update rule) is a famous example of regret-minimizing algorithms. It was introduced to the machine learning community by Freund and Schapire in~\cite{freund1999adaptive}, a generalization of the weighted majority algorithm of Littlestone and Warmuth~\cite{littlestone1989weighted}. The Hedge algorithm will be central in our discussion, as it will motivate the study of the continuous-time replicator equation, and will eventually be shown to converge for congestion games.

No-regret learning and its resulting population dynamics have been studied in the context of routing games, a special case of congestion games. For example, in~\cite{Blum2006routing}, Blum et al. show that the sequence of strategy profiles converges to the set of $\epsilon$-approximate Nash equilibria on a $(1-\epsilon)$-fraction of days. In other words, a subsequence of strategy profiles (in which an $\epsilon$ fraction of terms is dropped) converges to the set of $\epsilon$-approximate Nash equilibria. They also give explicit convergence rates which depend on the maximum slopes of the congestion functions.

Continuous-time population dynamics have also been studied for congestion games. In~\cite{fischer2004evolution}, Fischer and Vocking study the convergence of the replicator dynamics for the congestion game. The replicator ODE is also of particular interest in evolutionary game theory, see for example~\cite{weibull1997evolutionary}. In \cite{sandholm2001potential}, Sandholm studies convergence for the larger class of potential games. He shows that dynamics which satisfy a positive correlation condition with respect to the potential function of the game converge to the set of stationary points of the vector field (usually, a superset of Nash equilibria). However, many regret-minimizing algorithms do not satisfy this correlation condition. Our discussion is mainly concerned with discrete-time dynamics. However, properties of the replicator equation will be used in our analysis.

We will consider a model in which the losses are discounted over time, using a vanishing sequence of discount factors $(\gamma_\tau)_{\tau \in \Nbb}$, meaning that future losses matter less to players than present losses. This defines a discounted regret, and we will focus our attention on online learning algorithms with sublinear discounted regret. The sequence of discount factors will have several interpretations beyond its economic motivation. For example, we will observe that some multiplicative weight algorithms, such as the Hedge algorithm, have sublinear discounted regret if we use the sequence $(\gamma_\tau)_\tau$ as learning rates, provided it also satisfies $\sum_{\tau \leq T} {\gamma_\tau}^2 / \sum_{\tau \leq T} \gamma_\tau \rightarrow 0$ as $T \rightarrow \infty$.

After defining the model and giving preliminary results in Sections~\ref{sec:model} and~\ref{sec:learning}, we show in Section~\ref{sec:cesaro} that when players use online learning algorithms with sublinear discounted regret, the sequence of strategy profiles converges to the set of Nash equilibria in the \cesaro\ sense. In order to obtain strong convergence, we first motivate the study of the replicator dynamics. Indeed, it can be viewed as a continuous-time limit of the Hedge algorithm with decreasing learning rates. In Section~\ref{sec:replicator}, we recall the convergence result of the replicator dynamics. By discretizing the replicator equation (using the same discount sequence $(\gamma_\tau)_{\tau \in \Nbb}$ as discretization time steps) we obtain a multiplicative-weights update rule with sublinear discounted regret, which we call REP algorithm, for replicator. Finally, in Section~\ref{sec:conv}, we define a class of online learning algorithms we call the AREP algorithms, which can be expressed as a discrete REP algorithm with perturbations that satisfy a condition given in Definition~\ref{def:AREP}. Using results from the theory of stochastic approximation, we show that strong convergence is guaranteed for AREP algorithms with sublinear discounted regret. We finally observe that both the REP algorithm and the Hedge algorithm belong to this class, which proves convergence for these two algorithms in particular.

\section{The congestion game model}
\label{sec:model}
In the congestion game, a finite set $\Rcal$ of resources is shared by a set $\Xcal$ of players. The set of players is endowed with a structure of measure space, $(\Xcal, \Mcal, m)$, where $\Mcal$ is a $\sigma$-algebra of measurable subsets, and $m$ is a finite Lebesgue measure. The measure is non-atomic, in the sense that single-player sets are null-sets for $m$. The player set is partitioned into $K$ populations, $\Xcal = \Xcal_1 \cup \dots \cup \Xcal_K$. For all $k$, the total mass of population $\Xcal_k$ is assumed to be finite and nonzero. Each player $x \in \Xcal_k$ has a task to perform, characterized by a collection of bundles $\Pcal_k \subset \Pcal$, where $\Pcal$ is the power set of $\Rcal$. The task can be accomplished by choosing any bundle of resources $p \in \Pcal_k$. The action set of any player in $\Xcal_k$ is then simply $\Pcal_k$.

The joint actions of all players can be represented by an action profile $a: \Xcal \rightarrow \Pcal$ such that for all $x \in \Xcal_k$, $a(x) \in \Pcal_k$ is the bundle of resources chosen by player $x$. The function $x \mapsto a(x)$ is assumed to be $\Mcal$-measurable ($\Pcal$ is equipped with the counting measure). The action profile $a$ determines the bundle loads and resource loads, defined as follows: for all $k \in \{1, \dots, K\}$ and $p \in \Pcal_k$, the load of bundle $p$ under population $\Xcal_k$ is the total mass of players in $\Xcal_k$ who chose that bundle
\begin{equation}
\label{eq:bundle_load}
 f^k_p(a) = \int_{x \in \Xcal_k} \ind{a(x) = p} dm(x)
\end{equation}
For any $r \in \Rcal$, the resource load is defined to be the total mass of players utilizing that resource
\begin{equation}
\label{eq:resource_load}
\phi_r(a) = \sum_{k = 1}^K \sum_{p \in \Pcal_k : r \in p}  f^k_p(a)
\end{equation}
The resource loads determine the losses of all players: the loss associated to a resource $r$ is given by $c_r(\phi_r(a))$, where the congestion functions $c_r$ are assumed to satisfy the following:
\begin{assumption}
\label{ass:congestion_functions}
The congestion functions $c_r$ are non-negative, non-decreasing, Lipschitz-continuous functions. 
\end{assumption}

The total loss of a player $x$ such that $a(x) = p$ is $\sum_{r \in p} c_r(\phi_r(a))$.

The congestion model is given by the tuple $(K, (\Xcal_k)_{1 \leq k \leq K}, \Rcal, (\Pcal_k)_{1 \leq k \leq K}, (c_r)_{r \in \Rcal})$. The congestion game is determined by the action set and the loss function for every player: 
for all $x \in \Xcal_k$, the action set of $x$ is $\Pcal_k$, and the loss function of~$x$, given the action profile~$a$, is
\[
\sum_{p \in \Pcal} \ind{a(x) = p} \sum_{r \in p} c_r(\phi_r(a)).
\]

\subsection{A macroscopic view} The action profile $a$ specifies the bundle of each player~$x$. A more concise description of the joint action of players is given by the bundle distribution: the proportion of players choosing bundle $p$ in population $\Xcal_k$ is denoted by $\mu^k_p(a) = f^k_p(a) / m(\Xcal_k)$, which defines a bundle distribution for population $\Xcal_k$,
\[
\mu^k(a) = (\mu^k_p(a))_{p \in \Pcal_k} \in \Delta^{\Pcal_k},
\]
and a bundle distribution across populations, given by the product distribution
\[
\mu(a) = (\mu^1(a), \dots, \mu^K(a)) \in \Delta^{\Pcal_1} \times \dots \times \Delta^{\Pcal_K}.
\]
We say that the action profile $a$ induces the distribution $\mu(a)$. Here $\Delta^{\Pcal_k}$ denotes the simplex of distributions over $\Pcal_k$, that is
\[
\Delta^{\Pcal_k} = \left\{\mu \in \Rbb^{\Pcal_k}_+ : \sum_{p \in \Pcal_k} \mu_p = 1 \right\}
\]
The product of simplexes $\Delta^{\Pcal_1} \times \dots \times \Delta^{\Pcal_K}$ will be denoted $\Delta$. This macroscopic representation of the joint actions of players will be useful in our analysis. We will also view the resource loads as linear functions of the product distribution $\mu(a)$. Indeed, we have from equation~\eqref{eq:resource_load} and the definition of $\mu_p^k(a)$
\[
\phi_r(a) = \sum_{k = 1}^K m(\Xcal_k)\sum_{p \in \Pcal_k : r \in p} \mu^k_p(a) = \sum_{k = 1}^K m(\Xcal_k)(M^k \mu^k(a))_r
\]
where for all $k$, $M^k \in \Rbb^{\Rcal\times \Pcal_k}$ is an incidence matrix defined as follows: for all $r \in \Rcal$ and all $p \in \Pcal_k$,
\[
M^k_{r, p} = \begin{cases}
1 &\text{if $r \in p$} \\
0 &\text{otherwise}
\end{cases}
\]
We write in vector form $\phi(a) = \sum_{k = 1}^K m(\Xcal_k)M^k \mu^k(a)$, and by defining the scaled incidence matrix $\bar{M} = \mat{{c}m(\Xcal_1)M^1 | \dots | m(\Xcal_K)M^K}$, we have
\[
\phi(a) = \bar{M} \mu(a)
\]

By abuse of notation, the dependence on the action profile $a$ will be omitted, so we will write $\mu$ instead of $\mu(a)$ and $\phi$ instead of $\phi(a)$. Finally, we define the loss function of a bundle $p \in \Pcal_k$ to be
\begin{equation}
\label{eq:bundel_loss}
\ell^k_p(\mu) = \sum_{r \in p} c_r(\phi_r) = \sum_{r \in p} c_r((\bar{M}\mu)_r) = M^\top c(\bar{M}\mu)
\end{equation}
where $M$ is the incidence matrix $M = \mat{{ccc}M^1 | &\dots &| M^K}$, and $c(\phi)$ is the vector $\parenth{c_r(\phi_r)}_{r \in \Rcal}$. We denote by $\ell^k(\mu)$ the vector of losses $(\ell^k_p(\mu))_{p \in \Pcal_k}$, and by $\ell(\mu)$ the $K$-tuple $\ell(\mu) = (\ell^1(\mu), \dots, \ell^K(\mu))$.
\subsection{Nash equilibria of the congestion game}
We can now define and characterize the Nash equilibria of the congestion game, also called Wardrop equilibria, in reference to~\cite{wardrop1952road}.
\begin{definition}[Nash equilibrium]~\\
\label{def:nash}
A product distribution $\mu$ is a Nash equilibrium of the congestion game if for all $k$, and all $p \in \Pcal_k$ such that $\mu^k_p > 0$, $\ell^k_{p'}(\mu) \geq \ell^k_{p} (\mu)$ for all $p' \in \Pcal_k$. The set of Nash equilibria will be denoted by $\Ncal$.
\end{definition}

In finite player games, a Nash equilibrium is defined to be an action profile $a$ such that no player has an incentive to unilaterally deviate~\cite{nash1951non}, that is, no player can strictly decrease her loss by unilaterally changing her action. We show that this condition (referred to as the Nash condition) holds for \emph{almost all players} whenever~$\mu$ is a Nash equilibrium in the sense of Definition~\ref{def:nash}.
\begin{proposition}
\label{prop:nash_equivalent}
A distribution $\mu$ is a Nash equilibrium if and only if for any joint action $a$ which induces the distribution $\mu$, almost all players have no incentive to unilaterally deviate from $a$.
\end{proposition}
\begin{proof}
First, we observe that, given an action profile $a$, when a single player $x$ changes her strategy, this does not affect the distribution $\mu$. This follows from the definition of the distribution,
\[
\mu^k_p = \frac{1}{m(\Xcal_k)} \int_{\Xcal_k} \ind{a(x) = p} dm(x).
\]
Changing the action profile $a$ on a null-set $\{x\}$ does not affect the integral.

Now, assume that almost all players have no incentive to unilaterally deviate. That is, for all $k$, for almost all $x \in \Xcal_k$,
\begin{equation}
\label{eq:nash_condition}
\forall p' \in \Pcal_k, \  \ell^k_{p'} (\mu')  \geq \ell^k_{a(x)} (\mu)
\end{equation}
where $\mu'$ is the distribution obtained when $x$ unilaterally changes her bundle from $a(x)$ to $p'$. By the previous observation, $\mu' = \mu$.
As a consequence, condition~\eqref{eq:nash_condition} becomes: for almost all $x$, and for all $p'$, $\ell^k_{p'}(\mu) \geq \ell^k_{a(x)}(\mu)$. Therefore, integrating over the set $\{x \in \Xcal_k : a(x) = p\}$, we have for all $k$,
\begin{equation*}
\ell^k_{p'} ( \mu ) \mu^k_p  \geq \ell^k_{p} (\mu) \mu^k_p \text{ for all $p'$}.
\end{equation*}
which implies that $\mu$ is a Nash equilibrium in the sense of Definition~\ref{def:nash}. Conversely, if $a$ is an action profile, inducing distribution $\mu$, such that the Nash condition does not hold for a set of players with positive measure, then there exists $k_0$ and a subset $X \subset \Xcal_{k_0}$ with $m(X) > 0$, such that every player in $X$ can strictly decrease her loss by changing her action. Let $X_p = \{x \in X : a(x) = p\}$, then $X$ is the disjoint union $X = \cup_{p \in \Pcal_k} X_p$, and there exists $p_0$ such that $m(X_{p_0}) > 0$. Therefore
\[
\mu^{k_0}_{p_0} = \frac{m\parenth{\{x \in \Xcal_{k_0} : a(x) = p_0\} }}{m(\Xcal_{k_0})} \geq \frac{m(X_{p_0})}{m(\Xcal_{k_0})} > 0.
\]
Let $x \in X_{p_0}$. Since $x$ can strictly decrease her loss by unilaterally changing her action, there exists $p_1$ such that $\ell^{k_0}_{p_1}(\mu) < \ell^{k_0}_{a(x)}(\mu) = \ell^{k_0}_{p_0}(\mu)$. But since $\mu^{k_0}_{p_0} > 0$, $\mu$ is not a Nash equilibrium.
\end{proof}

Definition~\ref{def:nash} also implies that, for a population $\Xcal_k$, all bundles with non-zero mass have equal losses, and bundles with zero mass have greater losses. Therefore almost all players incur the same loss. This observation motivates a second characterization of Nash equilibria, in terms of the average loss.
\begin{definition}[Average loss]
\label{def:avg_loss}
The average loss incurred by population $\Xcal_k$ is the real number:
\[
\bar{\ell}^k(\mu) = \frac{1}{m(\Xcal_k)} \int _{\Xcal_k} \ell^k_{a(x)} (\mu) dm(x) = \sum _{p \in \Pcal_k} \mu^k_p \ell^k_p(\mu)
\]
\end{definition}

\begin{proposition} \label{prop:nash_average_latency} $\mu$ is a Nash equilibrium if and only if for all $k$ and all $p \in \Pcal_k$, $\ell^k_p(\mu) \geq \bar{\ell}^k(\mu)$.
\end{proposition}

\begin{proof} If $\mu$ is a Nash equilibrium, then all bundles with non-zero mass have equal losses, and bundles with zero mass have greater losses. That is, for all $k$, there exists $p_0 \in \Pcal_k$ such that for all $p$, if $\mu_{p} > 0$ then $\ell^k_{p}(\mu) = \ell^k_{p_0}(\mu)$, and if $\mu_{p} = 0$, then $\ell^k_{p}(\mu) \geq \ell^k_{p_0}(\mu)$. Thus
\[
\bar{\ell}^k(\mu) = \sum _{p \in \Pcal_k : \mu_{p} > 0} \mu_{p} \ell^k_{p}(\mu) = \parenth{\sum _{p \in \Pcal_k : \mu_{p} > 0} \mu_{p}} \ell^k_{p_0}(\mu) = \ell^k_{p_0}(\mu),
\]
and it follows that for all $p$, $\ell^k_p(\mu) \geq \ell^k_{p_0}(\mu) = \bar{\ell}^k(\mu)$.

Conversely, assume that for all $p$, $\ell^k_p(\mu) \geq \bar{\ell}^k(\mu)$, and let $p_0 \in \arg \min_{p \in \Pcal_k : \mu_{p} > 0} \ell_{p}(\mu)$. Then,
\[
\ell^k_{p_0}(\mu) \geq \bar{\ell}^k(\mu) = \sum _{p \in \Pcal_k : \mu_{p} > 0} \mu_{p} \ell^k_{p}(\mu) \geq \ell^k_{p_0}(\mu) \sum _{p \in \Pcal_k : \mu_{p} > 0} \mu_{p} = \ell^k_{p_0}(\mu)
\]
and the inequalities must hold with equality, thus all bundles with non-zero mass have the same loss (equal to the average loss), while bundles with zero mass have larger losses. This proves that $\mu$ is a Nash equilibrium.
\end{proof}

\subsection{Mixed strategies}
\label{subsec:mixed}
The Nash equilibria we have described so far are {\it pure strategy} equilibria, since each player $x$ deterministically plays a single action $a(x)$. 
We now extend the model to allow mixed strategies. That is, the action of a player $x$ is a random variable $A(x)$ with distribution $\pi(x)$, and with realization $a(x)$.

We show that when players use mixed strategies, provided they randomize independently, the resulting Nash equilibria are, in fact, the same as those given in Definition~\ref{def:nash}. The key observation is that under independent randomization, the resulting bundle distributions $\mu^k$ are random variables with zero variance, thus they are essentially deterministic.

To formalize the probabilistic setting, let $(\Omega, \Fcal, \Prob)$ be a probability space. A mixed strategy profile is a function $A: \Xcal \rightarrow \Omega \rightarrow \Pcal$, such that for all $k$ and all $x \in \Xcal_k$,  $A(x)$ is a $\Pcal_k$-valued random variable, such that the mapping $(x, \omega) \mapsto A(x)(\omega)$ is $\Mcal \times \Fcal$-measurable. For all $x \in \Xcal_k$ and $p \in \Pcal_k$, let $\pi^k_{p}(x) = \Prob[A(x) = p]$. Similarly to the deterministic case, the mixed strategy profile $A$ determines the bundle distributions $\mu^k$, which are, in this case, random variables, as we recall that:
\begin{equation*}
\mu^k_p = \frac{1}{m(\Xcal_k)} \int _{\Xcal_k} \ind{A(x) = p} dm(x)
\end{equation*}
Nevertheless, assuming players randomize independently, the bundle distribution is almost surely equal to its expectation, as stated in the following Proposition. The assumption of independent randomization is a reasonable one, since players are non-cooperative.

\begin{proposition}
\label{prop:distribution_deterministic}
Under independent randomization,
\[
\forall k, \text{almost surely, }\mu^k = \Exp[\mu^k] = \frac{1}{m(\Xcal_k)} \int_{\Xcal_k} \pi^k(x)dm(x)
\]
\end{proposition}
\begin{proof}
Fix $k$ and let $p \in \Pcal_k$. Since $(x, \omega) \mapsto \ind{A(x) = p}(\omega)$ is a non-negative bounded $\Mcal \times \Fcal$-measurable function, we can apply Tonelli's theorem and write:
\begin{align*}
\Exp \sqbr{ \mu^k_p }
&= \Exp \left[ \frac{1}{m(\Xcal_k)}\int_{\Xcal_k} \ind{A(x) = p} dm(x) \right] \\
&= \frac{1}{m(\Xcal_k)}\int_{\Xcal_k}\Exp \left[ \ind{A(x) = p} \right]dm(x) \\
&= \frac{1}{m(\Xcal_k)} \int_{\Xcal_k} \pi^k_p(x)dm(x)
\end{align*}
Similarly,
\begin{align*}
m(&\Xcal_k)^2 \var \sqbr{ \mu^k_p } \\
&= \Exp \parenth{ \int_{\Xcal_k} \ind{A(x) = p} dm(x)}^2 - \parenth{\int_{\Xcal_k} \pi^k_p(x)dm(x)}^2 \\
&= \int_{\Xcal_k} \int_{\Xcal_k} \Exp \ind{A(x) = p; A(x')=p} dm(x) dm(x') - \int_{\Xcal_k}\int_{\Xcal_k} \pi^k_p(x)\pi^k_p(x')dm(x)dm(x') \\
&= \int_{\Xcal_k \times \Xcal_k} \left( \Prob [A(x) = p; A(x')=p] - \pi^k_p(x)\pi^k_p(x') \right) d(m \times m)(x, x')
\end{align*}
Then observing that the diagonal $D = \{ (x, x) \colon x \in \Xcal_k\}$ is an $(m \times m)$-nullset (this follows for example from Proposition 251T in \cite{fremlin2000measure}), we can restrict the integral to the set $\Xcal_k \times \Xcal_k \setminus D$, on which $\Prob[A(x) = p; A(x')=p] = \pi^k_p(x)\pi^k_p(x')$, by the independent randomization assumption. This proves that $\var \sqbr{ \mu^k_p } = 0$. Therefore $\mu_p^k = \Exp { \mu^k_p }$ almost surely.
\end{proof}

We observe that here, the assumption of non-atomicity is essential. This fact is reminiscent of temperature in statistical physics: independent measurements of the global distribution of the same state yield the same result almost surely.



\subsection{The Rosenthal potential function}
\label{subsec:rosenthal}

We now discuss how one can formulate the set of Nash equilibria as the solution of a convex optimization problem. Consider the function
\begin{equation}
\label{eq:pot_lat}
V(\mu) = \sum _{r \in \Rcal} \int _0 ^{(\bar{M}\mu)_r} c_r(u) du
\end{equation}
defined on the product of simplexes $\Delta^{\Pcal_1} \times \dots \times\Delta^{\Pcal_K}$, which will be denoted $\Delta$. $V$ is called the Rosenthal potential function, and was introduced in~\cite{rosenthal1973class} for the congestion game with finitely many players, and later generalized to the infinite-players case.
It can be viewed as the composition of the function $\bar{V} : \phi \in \Rbb_+^{\Rcal} \mapsto \sum_{r \in \Rcal} \int_{0}^{\phi_r} c_r(u)du$ and the linear function $\mu \mapsto \bar{M}\mu$. 
Since for all $r$, $c_r$ is, by assumption, non-negative, $\bar{V}$ is differentiable, non-negative and $\nabla \bar{V}(\phi) = \parenth{c_r(\phi_r)}_{r \in \Rcal}$. And since $c_r$ are non-decreasing, $\bar{V}$ is convex. Therefore $V$ is convex as the composition of a convex and a linear function.

A simple application of the chain rule gives $\nabla V(\mu) = \bar{M}^\top c(\bar{M}\mu)$. If we denote~$\nabla_{\mu^k} V(\mu)$ the vector of partial derivatives with respect to $\mu^k_p$, $p \in \Pcal_k$, we have $\nabla_{\mu^k} V(\mu) = m(\Xcal_k) {M^k}^\top c(\bar{M}\mu) = m(\Xcal_k) \ell^k(\mu)$. Thus,
\begin{equation}
\label{eq:potential_derivative}
\forall k, \ \forall p \in \Pcal_k, \quad  \frac{\partial V} {\partial \mu^k_p} (\mu) = m(\Xcal_k) \ell^k_p(\mu)
\end{equation}
and $V$ is a potential function for the congestion game, as defined in~\cite{sandholm2001potential} for example. 


Next, we show the relationship between the set of Nash equilibria and the potential function $V$.

\begin{theorem}[Rosenthal~\cite{rosenthal1973class}]
\label{thm:rosenthal}
$\Ncal$ is the set of minimizers of $V$ on the product of simplexes $\Delta$. It is a non-empty convex compact set. We will denote $V_\Ncal$ the value of $V$ on~$\Ncal$.
\end{theorem}

A version of this theorem is proved in~\cite{rosenthal1973class}. We also give a proof in Appendix~\ref{app:rosenthal_proof}. 


Since the set of Nash equilibria can be expressed as the solution to a convex optimization problem, it can be computed in polynomial time in the size of the problem. Beyond computing Nash equilibria, we seek to model how players arrive at the set $\Ncal$. This is discussed in Section~\ref{sec:learning}. But first, we define routing games, a special case of congestion games.

\subsection{Example: routing games}
\label{subsec:routing_game}

A routing game is a congestion game with an underlying graph $\Gcal = (\Vcal, \Ecal)$, with vertex set $\Vcal$ and edge set $\Ecal \subset \Vcal \times \Vcal$. In this case, the resource set is equal to the edge set, $\Rcal = \Ecal$. Routing games are used to model congestion on transportation or communication networks. Each population $\Xcal_k$ is characterized by a common source vertex $s_k \in \Vcal$ and a common destination vertex $t_k \in \Vcal$. In a transportation setting, players represent drivers traveling from $s_k$ to $t_k$; in a communication setting, players send packets from $s_k$ to $t_k$. The action set $\Pcal_k$ is a set of paths connecting $s_k$ to $t_k$. In other words, each player chooses a path connecting his or her source and destination vertices. The bundle load $f^k_p$ is then called the flow on path $p$. The resource load $\phi_r$ is called the total edge flow. Finally, the congestion functions $\phi_r \mapsto c_r(\phi_r)$ determine the delay (or latency) incurred by each player.

\begin{figure}[h]
\centering
\includegraphics[width=.35\textwidth]{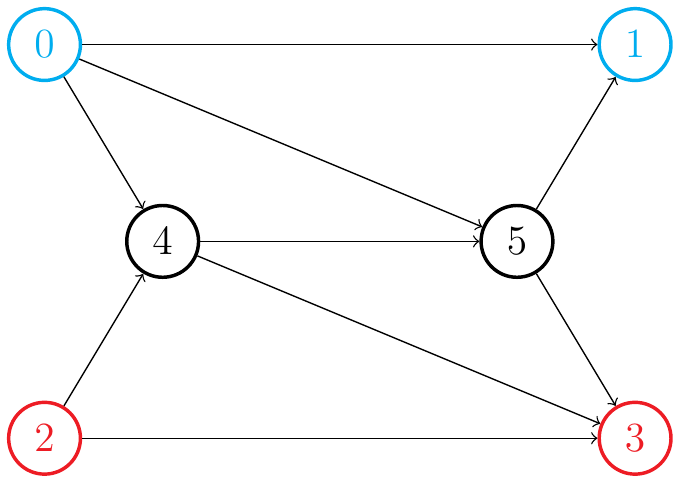}
\caption{Routing game with two populations of players.}
\label{fig:network}
\end{figure}

We will use the routing game given in Figure~\ref{fig:network} as an example to illustrate some of our results in later sections. In this example, two populations of players share the network, the first population sends packets from $v_0$ to $v_1$, and the second population from $v_2$ to $v_3$. The population masses are $F_1 = F_2 = 1$. The congestion functions are given below:
\begin{align*}
c_{(v_0, v_1)}(u) &= u+2 
& c_{(v_0, v_4)}(u) &= \frac{u}{2} 
& c_{(v_0, v_5)}(u) &= u \\
c_{(v_2, v_3)}(u) &= u+1 
& c_{(v_2, v_4)}(u) &= \frac{1}{2} 
& c_{(v_4, v_3)}(u) &= u \\
c_{(v_4, v_5)}(u) &= 3u
& c_{(v_5, v_1)}(u) &= \frac{u}{3} 
& c_{(v_5, v_3)}(u) &= \frac{u}{4}
\end{align*}

The paths (bundles) available to each population are given by:
\begin{align*}
\Pcal_1 &= \{ (v_0, v_1), (v_0, v_4, v_5, v_1), (v_0, v_5, v_1)\} \\
\Pcal_2 &= \{ (v_2, v_3), (v_2, v_4, v_5, v_3), (v_2, v_4, v_3) \}
\end{align*}
In this case, since the congestion functions are linear, the Rosenthal potential function is quadratic. Its minimizer is, in this example, unique, given by
\[
\Ncal = \big((0, 0.187, 0.813), (0.223, 0.053, 0.724) \big)
\]
and the corresponding path losses are given by
\begin{itemize}
\item for all $p \in \Pcal_1 \setminus (v_0, v_4, v_5, v_1)$, $\ell^1_p(\mu) = 1.14$
\item for $p = (v_0, v_4, v_5, v_1)$, $\ell^1_p(\mu) = 2.00$
\item for all $p \in \Pcal_2$, $\ell^2_p(\mu) = 1.22$
\end{itemize}

\section{Online learning in congestion games}
\label{sec:learning}

We now describe the online learning framework for the congestion game, and present the Hedge algorithm in particular.

\subsection{The online learning framework}
Suppose that the game is played repeatedly for infinitely many iterations, indexed by $\tau \in \Nbb$. During iteration $\tau$, each player chooses a bundle simultaneously. The decision of all players can be represented, as defined above, by an action profile $a^{(\tau)} : \Xcal \rightarrow \Pcal$. This induces, at the level of each population $\Xcal_k$, a bundle distribution ${\mu^k}^{(\tau)}$. These, in turn, determine the resource loads and the bundle losses $\ell^k_p(\mu^{(\tau)})$. The losses for bundles $p \in \Pcal_k$ are revealed to all players in population $\Xcal_k$, which marks the end of iteration $\tau$. Players can then use the information revealed to them to update their strategies before the start of the next iteration. 

\paragraph{A note on the information assumptions}
Here, we assume that at the end of the iteration, a player observes the losses of all bundles $p \in \Pcal_k$. Instead, one could assume that a player can only observe the losses she incurs. This is often called the multi-armed-bandit setting, in reference to the armed-bandit slot machines, in which a gambler can choose, at each iteration, one machine to play, and is only revealed the loss of that machine. Making this restriction requires players to use additional exploration of bundles. A comprehensive presentation of online learning algorithms in the multi-armed bandit setting, both stochastic and deterministic, can be found for example in~\cite{bubeck2011introduction, audibert2009minimax}. Regret bounds are also given in~\cite{cesa2006prediction} (Section 6.7, p.156-159) and~\cite{bubeck2009bounded,bubeck2012regret}. We choose to use the full feedback assumption to simplify our discussion, leaving the multi-armed-bandit setting as a possible extension. We believe this is a reasonable model in many games, since bundle losses could be announced publicly. In the special case of routing games, this can be achieved by having a central authority measure and announce the delays. This is particularly true in transportation networks, in which many agencies and online services measure delays and make this information publicly available. Assuming the full vector of bundle losses is revealed does not mean, however, that players have access to the individual resource loads $\phi_r^{(\tau)}$, or to the congestion functions $c_r(\cdot)$, which is consistent with our initial argument that, in a realistic model, players should only rely on the observed value of the bundle losses.

Each player $x \in \Xcal_k$ is assumed to draw her bundle from a randomized strategy $\pi^{(\tau)}(x) \in \Delta^{\Pcal_k}$ (the deterministic case is a special case in which $\pi^{(\tau)}(x)$ is a vertex on the simplex, i.e. a pure strategy). As discussed in Section~\ref{subsec:mixed}, players randomize independently. At the end of iteration $\tau$, player $x$ updates her strategy using an \emph{update rule} or \emph{online learning algorithm}, as defined below.

\begin{definition}[Online learning algorithm for the congestion game] An online learning algorithm (or update rule) for the congestion game, applied by a player $x \in \Xcal_k$, is a sequence of functions $\left(\tensor[^x]{U}{^{(\tau)}} \right)_{\tau \in \Nbb}$, fixed a priori, that is, before the start of the game, such that for each $\tau$,
\func{ \tensor[^x]{U}{^{(\tau)}} }{ \left( \Rbb^{\Pcal_k} \right)^\tau \times \Delta^{\Pcal_k}}{\Delta^{\Pcal_k}}{\left((\ell^k(\mu^{(t)}))_{t \leq \tau}, \pi^{(\tau)}(x) \right)}{\pi^{(\tau+1)}(x) }
is a function which maps, given the history of bundle losses $(\ell^k(\mu^{(t)}))_{t \leq \tau}$, the strategy on the current day $\pi^{(\tau)}(x)$ to the strategy on the next day $\pi^{(\tau+1)}(x)$.

\end{definition}

The online learning framework is summarized in Algorithm~\ref{alg:online}.
\begin{algorithm}[h]
\begin{small}
\caption{Online learning framework for the congestion game}
\label{alg:online}
\begin{algorithmic}[1]
\State For every player $x \in \Pcal_k$, an initial mixed strategy $\pi^{(0)}(x) \in \Simp^{\Pcal_k}$ and an online learning algorithm $(\tensor[^x]{U}{^{(\tau)}})_{\tau \in \Nbb}$
\For{each iteration $\tau \in \Nbb$}
\State Every player $x$ independently draws a bundle according to her strategy $\pi^{(\tau)}(x)$, i.e. $A^{(\tau)}(x) \sim \pi^{(\tau)}(x)$.
\State The vector of bundle losses $\ell^k(\mu^{(\tau)})$ is revealed to all players in $\Pcal_k$. Each player incurs the loss of the bundle she chose.
\State Players update their mixed strategies: $\pi^{(\tau+1)}(x) = \tensor[^x]{U}{^{(\tau)}} ((\ell^k_p(\mu^{(t)}))_{t \leq \tau}, \pi^{(\tau)}(x))$.
\EndFor
\end{algorithmic}
\end{small}
\end{algorithm}

We will focus our attention on algorithms which have vanishing upper bounds on the average discounted regret, defined in the next section.

\subsection{Discounted regret}
Since the game is played for infinitely many iterations, we assume that the losses of players are discounted over time. This is a common technique in infinite-horizon optimal control for example, and can be motivated from an economic perspective by considering that losses are devalued over time. We also give an interpretation of discounting in terms of learning rates, as discussed in Section~\ref{subsec:hedge}.

Let $(\gamma_\tau)_{\tau \in \Nbb}$ denote the sequence of discount factors. We make the following assumption:
\begin{assumption}
\label{assumption:discounts}
The sequence of discount factors $(\gamma_\tau)_{\tau \in \Nbb}$ is assumed to be positive decreasing, with $\lim_{\tau \rightarrow \infty} \gamma_\tau = 0$ and $\lim_{T \rightarrow \infty} \sum_{\tau = 0}^T \gamma_\tau = \infty$.
\end{assumption}

On iteration $\tau$, a player $x \in \Xcal_k$ who draws an action $A^{(\tau)}(x) \sim \pi^{(\tau)}(x)$ incurs a discounted loss given by $\gamma_\tau\ell^k_{A^{(\tau)}(x)}(\mu^{(\tau)})$, where $\mu^{(\tau)}$ is the distribution induced by the profile $A^{(\tau)}$. The cumulative discounted loss for player $x$, up to iteration $T$, is then defined to be
\begin{equation}
\label{eq:cumulative_loss}
{L}{^{(T)}}(x) 
= \sum_{\tau = 0}^T \gamma_\tau \ell^k_{A^{(\tau)}(x)}(\mu^{(\tau)})  
\end{equation}
We observe that this is a random variable, since the action $A^{(\tau)}(x)$ of player $x$ is random, drawn from a distribution $\pi^{(\tau)}(x)$. The expectation of the cumulative discounted loss is then
\begin{align*}
\Exp[ {L}{^{(T)}}(x) ] 
&= \sum_{\tau = 0}^T \gamma_\tau \Exp\left[ \ell^k_{A^{(\tau)}(x)}(\mu^{(\tau)}) \right] \\
&= \sum_{\tau = 0}^T \gamma_\tau
\braket{\pi^{(\tau)}(x)}{\ell^k(\mu^{(\tau)})}
\end{align*}
where $\braket{\cdot}{\cdot}$ denotes the Euclidean inner product on $\Rbb^{\Pcal_k}$. Similarly, we define the cumulative discounted loss for a fixed bundle $p \in \Pcal_k$
\begin{equation}
\label{eq:cumulative_bundle_loss}
{\Lscr_p^k}^{(T)} = \sum_{\tau = 0}^T \gamma_\tau\ell_p^k(\mu^{(\tau)})
\end{equation}
We can now define the discounted regret.
\begin{definition}[Discounted regret] Let $x \in \Xcal_k$, and consider an online learning algorithm for the congestion game, given by the sequence of functions $\left(\tensor[^x]{U}{^{(\tau)}} \right)_{\tau \in \Nbb}$. Let $(\mu^{(\tau)})_{\tau \in \Nbb}$ be the sequence of distributions, determined by the mixed strategy profile of all players. Then the discounted regret up to iteration $T$, for player $x$, under algorithm $U$, is the random variable
\begin{equation}
\label{eq:regret}
{R}{^{(T)}}(x) = {L}{^{(T)}}(x) - \min_{p \in \Pcal_k} {\Lscr_p^k}^{(T)}
\end{equation}
The algorithm $U$ is said to have sublinear discounted regret if, for any sequence of distributions $(\mu^{(\tau)})_{\tau \in \Nbb}$, and any initial strategy $\pi^{(0)}$,
\begin{equation}
\frac{1} {\sum _{\tau = 0}^T \gamma_\tau} \left[ {R}{^{(T)}}(x) \right]^+ \rightarrow 0 \text{ almost surely as } T \to \infty
\end{equation}
If we have convergence in the $L^1$-norm, $\frac{1}{\sum_{\tau = 0}^T \gamma_\tau} \sqbr{\Exp \sqbr{ R^{(T)}(x)} }^+ \rightarrow 0$, we say that the algorithm has sublinear discounted regret \emph{in expectation}.
\end{definition}

We observe that, in the definition of the regret, one can replace the minimum over the set ${\Pcal_k}$ by a minimum over the simplex $\Delta^{\Pcal_k}$
\[
\min_{p \in \Pcal_k} L_p^{(T)} = \min_{\pi \in \Delta^{\Pcal_k}} \braket{\pi}{L^{(T)}}
\]
since the minimizers of a bounded linear function lie on the set of extremal points of the feasible set. Therefore, the discounted regret compares the performance of the online learning algorithm to the \emph{best constant strategy in hindsight}. Indeed, $\braket{\pi}{L^{(T)}}$ is the cumulative discounted loss of a constant strategy~$\pi$, and minimizing this expression over $\pi \in \Delta^{\Pcal_k}$ yields the best constant strategy in hindsight: one cannot know a priori which strategy will minimize the expression, until all losses up to $T$ are revealed. If the algorithm has sublinear regret, its average performance is, asymptotically, as good as the performance of any constant strategy, regardless of the sequence of distributions $(\mu^{(\tau)})_{\tau\in \Nbb}$.

\paragraph{A note on monotonicity of the discount factors:} A similar definition of discounted regret is used for example by Cesa-Bianchi and Lugosi in Section~3.2 of~\cite{cesa2006prediction}. However, in their definition, the sequence of discount factors is \emph{increasing}. This can be motivated by the following argument: present observations may provide better information than past, stale observations. While this argument is accurate in many applications, it does not serve our purpose of convergence of population strategies. In our discussion, the standing assumption is that discount factors are \emph{decreasing}.

Finally, we observe that the cumulative discounted loss and regret are bounded, uniformly in $x$.
\begin{proposition}
\label{prop:bound}
There exists $\rho \geq 0$ such that $\forall k$,
\begin{align}
&\forall p \in \Pcal_k, \ \forall \mu \in \Delta, \ \ell^k_p(\mu) \in [0, \rho] \label{bounded_loss}\\
&\forall x \in \Xcal_k, \ \frac{1}{\sum_{\tau = 0}^T \gamma_\tau} {L}{^{(T)}}(x) \in [0, \rho] \label{eq:bounded_cum_loss}\\
&\forall x \in \Xcal_k, \ \frac{1}{\sum_{\tau = 0}^T \gamma_\tau} \sqbr{ {R}{^{(T)}}(x) }^+ \in [0, \rho] \label{eq:bounded_regret}
\end{align}
\end{proposition}
\begin{proof}
Since the bundle loss functions $\mu \mapsto \ell^k_p(\mu)$ are continuous on the compact set $\Delta$, they are bounded, and since there are finitely many bundles, there exists a common bound $\rho$ such that for all $k$, for all $p \in \Pcal_k$ and all $\mu$, $0 \leq \ell^k_p(\mu) \leq \rho$. The bounds~\eqref{eq:bounded_cum_loss} and~\eqref{eq:bounded_regret} follow from~\eqref{bounded_loss} and the definitions~\eqref{eq:cumulative_loss} and~\eqref{eq:regret} of ${L}{^{(T)}}(x)$ and ${R}{^{(T)}}(x)$.
\end{proof}
\subsection{Population-wide regret}
We have defined the discounted regret ${R}{^{(T)}}(x)$ for a single player $x$. In order to analyze the population dynamics, we define a population-wide cumulative discounted loss ${L^k}^{(T)}$, and discounted regret ${R^k}^{(T)}$ as follows:
\begin{align}
{L^k}^{(T)} &= \frac{1}{m(\Xcal_k)} \int_{\Xcal_k} {L}{^{(T)}}(x)dm(x) \\
{R^k}^{(T)} &= \frac{1}{m(\Xcal_k)} \int_{\Xcal_k} {R}{^{(T)}}(x)dm(x) = {L^k}^{(T)} - \min_{p \in \Pcal_k} {\Lscr^k_p}^{(T)}
\end{align}
Since ${L}{^{(T)}}(x)$ is random for all $x$, ${L^k}^{(T)}$ is also a random variable. However, it is, in fact, almost surely equal to its expectation. Indeed, recalling that ${\mu_p^k}^{(\tau)}$ is the proportion of players who chose bundle $p$ at iteration $\tau$ (also a random variable), we can write
\begin{align}
{L^k}^{(T)} 
&= \sum_{\tau = 0}^T \gamma_\tau \frac{1}{m(\Xcal_k)}\int_{\Xcal_k} \ell^k_{A^{(\tau)}(x)}(\mu^{(\tau)}) dm(x) \notag \\
&= \sum_{\tau = 0}^T \gamma_\tau \frac{1}{m(\Xcal_k)} \sum_{p \in \Pcal_k} \int_{ \{ x \in \Xcal_k: A^{(\tau)}(x) = p \} } \ell^k_{p}(\mu^{(\tau)}) dm(x) \notag \\
&= \sum_{\tau = 0}^T \gamma_\tau \sum_{p \in \Pcal_k} {\mu^k_p}^{(\tau)} \ell^k_p(\mu^{(\tau)})
\end{align}
thus assuming players randomize independently, $\mu^{(\tau)}$ is almost surely deterministic by Proposition~\ref{prop:distribution_deterministic}, and so is ${L^k}^{(T)}$. The same holds for ${R^k}^{(T)}$.

\begin{proposition}
\label{prop:population_regret_sublinear}
If almost every player $x \in \Xcal_k$ applies an online learning algorithm with sublinear regret in expectation, then the population-wide regret is also sublinear.
\end{proposition}

\begin{proof}
By the previous observation, we have, almost surely,
\[
{R^k}^{(T)} = \Exp\left[ {R^k}^{(T)} \right] = \frac{1}{m(\Xcal_k)} \int_{\Xcal_k} \Exp\sqbr{ {R}{^{(T)}}(x) } dm(x)
\]
where the second equality follows from Tonelli's theorem. Taking the positive part and using Jensen's inequality, we have
\[
\frac{1}{\sum_{\tau = 0}^T \gamma_\tau} \sqbr{{R^k}^{(T)}}^+ \leq \frac{1}{m(\Xcal_k)} \int_{\Xcal_k} \frac{1}{\sum_{\tau = 0}^T \gamma_\tau}\sqbr{ \Exp\sqbr{{R}{^{(T)}}(x) } }^+ dm(x)
\]
By assumption, $\frac{1}{\sum_{\tau = 0}^T \gamma_\tau} \sqbr{\Exp \sqbr{{R}{^{(T)}}(x) }}^+$ converges to $0$ for all $x$, and by Proposition~\ref{prop:bound}, it is bounded uniformly in $x$. Thus the result follows by applying the dominated convergence theorem.
\end{proof}

\subsection{Hedge algorithm with vanishing learning rates}
\label{subsec:hedge}
We now present one particular online learning algorithm with sublinear regret. Consider a congestion game, and let $\rho$ be an upper bound on the losses. The existence of such an upper bound was established in Proposition~\ref{prop:bound}.
\begin{definition}[Hedge algorithm] \label{def:hedge} The Hedge algorithm, applied by player $x \in \Xcal_k$, with initial distribution $\pi^{(0)} \in \Delta^{\Pcal_k}$ and learning rates $(\eta_\tau)_{\tau \in \Nbb}$ is an online learning algorithm $(\tensor[^x]{U}{^{(\tau)}})_{\tau \in \Nbb}$ such that the $\tau$-th update function is given by
\[
\tensor[^x]{U}{^{(\tau)}} ( (\ell^k(\mu^{(t)}))_{t \leq \tau}, \pi^{(\tau)} ) = \psi \parenth{\parenth{ \pi_p^{(\tau)} \exp\parenth{- \eta_\tau \frac{\ell^k_p(\mu^{(\tau)}) }{ \rho } } }_{p \in \Pcal_k}}
\]
where $\psi$ is the normalization function
\func{\psi}{\Rbb_+^{\Pcal_k} \setminus \{ 0 \}}{\Delta^{\Pcal_k}}{v}{\frac{v}{\sum _{p \in \Pcal_k} v_p}}

That is, the distribution at iteration $\tau +1$ is proportional to the following vector
\begin{equation}
\label{eq:hedge_update}
\pi^{(\tau+1)} \propto \parenth{ \pi_p^{(\tau)} \exp\parenth{- \eta_\tau \frac{\ell^k_p(\mu^{(\tau)}) }{ \rho } } }_{p \in \Pcal_k}
\end{equation}
\end{definition}

Intuitively, the Hedge algorithm updates the distribution by computing, at each iteration, a set of bundle weights, then normalizing the vector of weights. The weight of a bundle $p$ is obtained by multiplying the probability at the previous iteration, $\pi^{(\tau)}_p$, by a term which is exponentially decreasing in the bundle loss $\ell^k_p(\mu^{(\tau)})$, thus the higher the loss of bundle $p$ at iteration $\tau$, the lower the probability of selecting $p$ at the next iteration. The parameter $\eta_\tau$ can be interpreted as a learning rate, as discussed in the following proposition.

\begin{proposition}
\label{prop:hedge_reg_div}
The Hedge update rule~\eqref{eq:hedge_update} is the solution to the following optimization problem:
\begin{equation}
\label{eq:hedge_reg_div}
\pi^{(\tau + 1)} \in \arg \min _{\pi \in \Simp^{\Pcal_k}} 
\braket{ \pi }{ \frac{\ell^k(\mu^{(\tau)})}{\rho} }
+ \frac{1}{\eta_\tau} D_{\text{KL}}(\pi \| \pi^{(\tau)})
\end{equation}
where $D_{\text{KL}}(\pi \| \nu) = \sum_{p \in \Pcal_k} \pi_p \log \frac{\pi_p}{\nu_p}$ is the Kullback-Leibler divergence of distribution $\pi$ with respect to $\nu$.
\end{proposition}

\begin{proof} Consider the Lagrangian of the problem, with dual variable $\lambda \in \Rbb$ associated to the constraint $\sum_{p \in \Pcal_k} \pi_p = 1$,
\begin{equation*}
\Lcal (\pi ; \lambda ) = \sum _{p \in \Pcal_k} \pi_p \frac{\ell^k(\mu^{(\tau)}) }{ \rho} + \frac{1}{\eta_\tau} \sum _{p \in \Pcal_k} \pi_p \log \frac{\pi_p}{\pi^{(\tau)}_p} + \lambda \left( \sum_{p \in \Pcal_k} \pi_p - 1 \right),
\end{equation*}
its gradient is given by
\begin{align*}
\frac{\partial}{\partial \pi_p} \Lcal (\pi ; \lambda) &= \frac{\ell^k_p(\mu^{(\tau)})}{\rho} + \frac{1}{\eta_\tau} \left(\log \frac{\pi_p}{\pi_p^{(\tau)}} + 1 \right) + \lambda \\
\frac{\partial}{\partial \lambda} \Lcal (\pi ; \lambda) &= \sum _{p \in \Pcal_k} \pi_p -1
\end{align*}
and $(\pi^\star, \lambda^\star)$ are primal-dual optimal if and only if the gradient of $\Lcal$ vanishes at $(\pi^\star, \lambda^\star)$, that is,
\al{
&\pi^*_p = \pi_p^{(\tau)}\exp\parenth{-1 - \eta_\tau \lambda - \eta_\tau \frac{\ell^k(\mu^{(\tau)}) }{ \rho}} \\
&\sum_{p \in \Pcal_k} \pi^\star_p = 1
}
which can be rewritten as $\pi^\star_p = \frac{1}{\alpha} \pi_p^{(\tau)} \exp\parenth{-\eta_\tau \frac{\ell^k(\mu^{(\tau)})}{\rho}}$, with $\alpha = \exp\parenth{1+\eta_\tau\lambda} = \sum_{p' \in \Pcal_k} \pi^{(\tau)}_{p'} \alpha\exp\parenth{-\eta \frac{\ell^k(\mu^{(\tau)})}{\rho}}$ is the normalization constant. Thus $\pi^\star$ satisfies the Hedge update equation~\eqref{eq:hedge_update}.
\end{proof}

The objective function in~\eqref{eq:hedge_reg_div} is the sum of an instantaneous loss term $\braket{\pi}{\frac{\ell^k(\mu^{(\tau)})}{\rho}}$ and a regularization term $\frac{1}{\eta_\tau} D_{KL}(\pi \| \pi^{(\tau)})$ which penalizes deviations from the previous distribution $\pi^{(\tau)}$, with a regularization coefficient $\frac{1}{\eta_\tau}$. The greedy problem (with no regularization term) would yield a pure strategy which concentrates all the mass on the bundle which had minimal loss on the previous iteration. With the regularization term, the player ``hedges her bet'' by penalizing too much deviation from the previous distribution. The coefficient $\eta_\tau$ determines the relative importance of the two terms in the objective function. In particular, as $\eta_\tau \rightarrow 0$, the solution to the problem~\eqref{eq:hedge_reg_div} converges to $\pi^{(\tau)}$ since the regularization term dominates the instantaneous loss term. In other words, as $\eta_\tau$ converges to $0$, the player stops learning from new observations, which justifies calling $\eta_\tau$ a \emph{learning rate}.

\begin{rem} \label{rem:mul} The sequence of distributions given by the Hedge algorithm also satisfy, for all $\tau$,
\begin{equation}
\label{eq:hedge_update_multiplicative}
\pi^{(\tau + 1)} \propto \parenth{ \pi_p^{(0)} \exp\parenth{- \sum_{t =0}^\tau \eta_t \frac{\ell^k_p(\mu^{(t)}) }{ \rho } } }_{p \in \Pcal_k} 
\end{equation}
\end{rem}

This follows from the update equation~\eqref{eq:hedge_update} and a simple induction on $\tau$. In particular, when $\eta_\tau = \gamma_\tau$, the term $\sum_{t = 0}^\tau \eta_t \ell^k_p(\mu^{(t)})$ coincides with the cumulative discounted loss ${\Lscr^k_p}^{(\tau)}$ defined in~\eqref{eq:cumulative_bundle_loss}. This motivates using the discount factors $\gamma_\tau$ as learning rates. We discuss this in the next proposition.

\begin{proposition}
\label{prop:hedge_bound}
Consider a congestion game with a sequence of discount factors $(\gamma_\tau)_{\tau \in \Nbb}$ satisfying Assumption~\ref{assumption:discounts}. Then the Hedge algorithm with learning rates $(\gamma_\tau)$ satisfies the following regret bound: for any sequence of distributions $(\mu^{(\tau)})_\tau$ and any initial strategy $\pi^{(0)}$,
\begin{equation*}
\Exp[{R}{^{(T)}}(x)] \leq -\rho \log \pi^{(0)}_{\min} + \frac{\rho}{8} \sum _{\tau = 0}^T \gamma_\tau^2,
\end{equation*}
where $\pi_{\min}^{(0)} = \min_{p \in \Pcal_k} \pi^{(0)}_p$. In particular, when $\frac{\sum_{\tau \leq T} \gamma_\tau^2}{\sum_{\tau \leq T} \gamma_\tau} \rightarrow 0$, the Hedge algorithm with rates $(\gamma_\tau)$ has sublinear discounted regret in expectation.
\end{proposition}

\begin{proof}
Given an initial strategy $\pi^{(0)}$, define \inlinefunc{\xi}{\Rbb^{\Pcal_k}}{\Rbb}{u}{\log \parenth{ \sum_{p \in \Pcal_k} \pi^{(0)}_p \exp (- \frac{u_p}{\rho} )} }.
Recalling the expression of the cumulative bundle loss ${\Lscr^k_p}^{(\tau)} = \sum_{t = 0}^\tau \gamma_t \ell^k_p(\mu^{(t)})$, we have for all $\tau \geq 0$:
\begin{align*}
\xi ({\Lscr^k}^{(\tau +1)}) - \xi({\Lscr^k}^{(\tau)})
&= \log \left( \sum _{p \in \Pcal_k} \frac{ \pi^{(0)}_p\exp\parenth{-\frac{{\Lscr_p^k}^{(\tau)}}{\rho}}  }{ \sum _{p' \in \Pcal_k} \exp\parenth{- \frac{ {\Lscr_{p'}^k}^{(\tau)} }{ \rho } } } \exp\parenth{-\gamma_{\tau+1} \frac{\ell^k_p(\mu^{(\tau+1)}) }{ \rho } } \right) \\
&= \log \parenth{ \sum _{p \in \Pcal_k} \pi_p^{(\tau+1)} \exp\parenth{-\gamma_{\tau+1} \frac{ \ell^k_p(\mu^{(\tau+1)}) }{ \rho }} } \\
& \leq - \gamma_{\tau+1} \sum _{p \in \Pcal_k} \pi_p^{(\tau+1)} \frac{\ell^k_p(\mu^{(\tau+1)}) }{ \rho } + \frac{\gamma_{\tau+1}^2 }{ 8 }
\end{align*}
The last inequality follows from Hoeffding's lemma (see Appendix~\ref{app:hoeffding}), since $0 \leq \frac{\ell^k_p(\mu^{(\tau)}) }{ \rho } \leq 1$. Summing over $\tau \in \{0, \ldots, T-1 \}$, we have for all $p$:
\[
\xi({\Lscr^k}^{(T)}) - \xi({\Lscr^k}^{(0)}) \leq -\sum_{\tau = 1}^T \gamma_\tau \sum _{p \in \Pcal^k}  \pi_p^{(\tau)} \frac{ \ell^k_p(\mu^{(\tau)}) }{ \rho } + \frac{1}{8} \sum_{\tau = 1}^T \gamma_\tau^2
\]
But we also have
\al{
\xi({\Lscr^k}^{(0)}) 
&= \log \parenth{\sum_{p \in \Pcal_k} \pi^{(0)}_p \exp\parenth{ - \gamma_0 \frac{ \ell^k(\mu^{(0)} }{ \rho } } } 
\leq -\gamma_0 \sum_{p \in \Pcal_k} \pi_p^{(0)} \frac{\ell^k_p(\mu^{(0)})}{\rho} + \frac{\gamma_0^2}{8}
}
And as $\log$ is increasing, we have for all $p_0 \in \Pcal_k$, $\log(\pi_{p_0}^{(0)}\exp(-\frac{{\Lscr^k_{p_0}}^{(T)}}{\rho})) \leq \xi({\Lscr^k}^{(T)})$, thus
\[
- \frac{ {\Lscr_{p_0}^k}^{(\tau)} }{ \rho } + \log \pi^{(0)}_{p_0} \leq \xi({\Lscr^k}^{(T)}) \leq -\sum_{\tau = 0}^T \gamma_\tau \sum _{p \in \Pcal^k}  \pi_{p}^{(\tau)} \frac{ \ell^k_{p}(\mu^{(\tau)}) }{ \rho } + \frac{1}{8} \sum_{\tau = 0}^T \gamma_\tau^2
\]
Rearranging, we have for all $p \in \Pcal_k$
\begin{equation*}
\sum _{\tau =0}^T \gamma_\tau \sum_{p \in \Pcal_k}  \pi_{p}^{(\tau)} \ell^k_{p}(\mu^{(\tau)}) - {\Lscr_{p_0}^k}^{(T)} \leq -\frac{\rho}{8} \log \pi_{p_0}^{(0)} + \rho \sum _{\tau = 0}^T \gamma_\tau^2
\end{equation*}
and we obtain the desired inequality by maximizing both sides over $p_0 \in \Pcal_k$.
\end{proof}

The previous proposition provides an upper-bound on the expected regret of the Hedge algorithm, of the form
\[
\frac{\Exp\sqbr{R^{(T)}(x)} }{ \sum_{\tau \leq T} \gamma_\tau} \leq -\rho\pi^{(0)}_{\min} \frac{1}{\sum_{\tau \leq T} \gamma_\tau} + \frac{\rho}{8} \frac{ \sum_{\tau \leq T}\gamma_\tau^2 }{ \sum_{\tau \leq T}\gamma_\tau }
\]
Given Assumption~\ref{assumption:discounts} on the discount factors, we have $\lim_{T \rightarrow \infty} \frac{\sum_{\tau \leq T} \gamma_\tau^2}{\sum_{\tau \leq T} \gamma_\tau} = 0$ (see Fact~\ref{fact:gamma_squares} in the Appendix), which proves that the discounted regret is sub-linear. This also provides a bound on the convergence rate. For example, if $\gamma_\tau \sim \frac{1}{\tau}$, then the upper-bound is equivalent to $\frac{c}{\log T}$, converges to zero as $T \rightarrow \infty$, albeit slowly. A better bound can be obtained for sequences of discount factors which are not square-summable, for example, taking $\gamma_\tau \sim \frac{1}{\sqrt{\tau}}$, the upper-bound is equivalent to $\frac{c\log T}{T^{\frac{1}{2}}}$.

We now have one example of an online learning algorithm with sublinear discounted regret. Furthermore, we have an interpretation of the sequence $\gamma_\tau$ as learning rates, which provides additional intuition on Assumption~\ref{assumption:discounts} on $(\gamma_\tau)$: decreasing the learning rates will help the system converge. 

In the next section, we start our analysis of the population dynamics when all players apply a learning algorithm with sublinear discounted regret.


\section{Convergence in the \cesaro\ sense}
\label{sec:cesaro}
As discussed in Proposition~\ref{prop:population_regret_sublinear}, if almost every player applies an algorithm with sublinear discounted regret in expectation, then the population-wide discounted regret is sublinear (almost surely). We now show that whenever the population has sublinear discounted regret, the sequence of distributions $(\mu^{(\tau)})_\tau$ converges in the sense of \cesaro. That is, $\sum_{\tau \leq T} \gamma_\tau \mu^{(\tau)} / \sum_{\tau \leq T} \gamma_\tau$ converges to the set of Nash equilibria. We also show that we have convergence of a dense subsequence. First, we give some definitions.

\begin{definition}[Convergence in the sense of \cesaro]
Fix a sequence of positive weights $(\gamma_\tau)_{\tau \in \Nbb}$. A sequence $(u^{(\tau)})_{\tau \in \Nbb}$ of elements of a normed vector space $(F, \| \cdot \|)$ converges to $u \in F$ in the sense of \cesaro\ means with respect to $(\gamma_\tau)_\tau$ if
\begin{equation*}
\lim _{T \to \infty} \frac{\sum_{\tau \in \Nbb : \tau \leq T} \gamma_\tau u^{(\tau)}}{\sum_{\tau \in \Nbb : \tau \leq T} \gamma_\tau} = u.
\end{equation*}
We write $u^{(\tau)} \toCes u$.
\end{definition}

The Stolz-\cesaro\ theorem states that if $(u^{(\tau)})_\tau$ converges to $u$, then it converges in the sense of \cesaro\ means with respect to any non-summable sequence $(\gamma_\tau)_\tau$, see for example~\cite{muresan2009concrete}. The converse is not true in general. However, if a sequence converges \emph{absolutely} in the sense of \cesaro\ means, i.e. $\| u^{(\tau)} - u \| \toCes 0$,
then a dense subsequence of $(u^{(\tau)})_\tau$ converges to $u$. To show this, we first show that absolute \cesaro\ convergence implies statistical convergence, as defined below.

\begin{definition}[Statistical convergence] Fix a sequence of positive weights $(\gamma_\tau)_\tau$. A sequence $(u^{(\tau)})_{\tau \in \Nbb}$ of elements of a normed vector space $(F, \| \cdot \|)$ converges to $u \in F$ statistically with respect to $(\gamma_\tau)$ if for all $\epsilon > 0$, the set of indexes $\Ical_\epsilon = \{ \tau \in \Nbb \colon \| u^{(\tau)} - u \| \geq \epsilon \}$ has zero density with respect to $(\gamma_\tau)$. The density of a subset of integers $\Ical \subset \Nbb$, with respect to the sequence of positive weights $(\gamma_\tau)$, is defined to be the limit, if it exists
\begin{equation*}
\lim _{T \to \infty} \frac{\sum _{\tau \in \Ical : \tau \leq T} \gamma_\tau} {\sum _{\tau \in \Nbb : \tau \leq T} \gamma_\tau}.
\end{equation*}
\end{definition}

\begin{lemma} \label{lem:2}
If $(u^{(\tau)})_\tau$ converges to $u$ absolutely in the sense of \cesaro\ means with respect to $(\gamma_\tau)$, then it converges to $u$ statistically with respect to $(\gamma_\tau)$.
\end{lemma}
\begin{proof} Let $\epsilon > 0$. We have for all $T \in \Nbb$,
\begin{equation*}
0 \leq \frac{\sum_{\tau \in \Ical_\epsilon \colon \tau \leq T} \gamma_\tau \epsilon} {\sum _{\tau \in \Nbb \colon \tau \leq T} \gamma_\tau} \leq \frac{\sum _{\tau \in \Nbb : \tau \leq T} \gamma_\tau \| u^{(\tau)} - u \| }{\sum _{\tau \in \Nbb : \tau \leq T} \gamma_\tau}
\end{equation*}
which converges to $0$ since $(u^{(\tau)})_\tau$ converges to $u$ absolutely in the sense of \cesaro\ means. Therefore
 $\Ical_\epsilon$ has zero density for all $\epsilon$.
\end{proof}

We can now show convergence of a dense subsequence.

\begin{proposition}
\label{prop:cesaro_implies_dense_subseq}
If $(u^{(\tau)})_{\tau \in \Nbb}$ converges to $u$ absolutely in the sense of \cesaro\ means with respect to $(\gamma_\tau)$, then there exists a subset of indexes $\Tcal \subset \Nbb$ of density one, such that the subsequence $(u^{(\tau)})_{\tau \in \Tcal}$ converges to $u$.
\end{proposition}
\begin{proof}
By Lemma~\ref{lem:2}, for all $\epsilon > 0$, the set $\Ical_\epsilon = \{ \tau \in \Nbb \colon \| u^{(\tau)} - u \| \geq \epsilon \}$ has zero density. We will construct a set $\Ical \subset \Nbb$ of zero density, such that the subsequence $(u_\tau)_{\tau \in \Nbb \setminus \Ical}$ converges. For all $k \in \Nbb^*$, let
\[
p_k(T) = \sum _{\tau \in \Ical_{\frac{1}{k}} \colon \tau \leq T} \gamma_\tau
\]
Since $\frac {p_{k}(T)}{\sum _{\tau \in \Nbb \colon \tau \leq T} \gamma_\tau}$ converges to $0$ as $T \to \infty$, there exists $T_k > 0$ such that for all $T \geq T_k$, $\frac {p_{k}(T)}{\sum _{\tau \in \Nbb \colon \tau \leq T} \gamma_\tau} \leq \frac{1}{k}$. Without loss of generality, we can assume that $(T_k)_{k \in \Nbb^*}$ is increasing. Now, let
\[
\Ical =\bigcup _{k \in \Nbb^*} ( \Ical_{\frac{1}{k}} \cap \{ T_{k}, \ldots, T_{k+1}-1 \} ).
\]
Then we have for all $k \in \Nbb^*$, $\Ical \cap \{0, \ldots, T_{k+1} - 1\} = \parenth{\cup_{j = 1}^{k} \Ical_{\frac{1}{j}} } \cap \{0, \ldots, T_{k+1}-1\}$. But since $\Ical_1 \subset \Ical_{\frac{1}{2}} \subset \dots \subset \Ical_{\frac{1}{k}}$, we have $\Ical \cap \{0, \dots, T_{k+1}-1\} \subset \Ical_{\frac{1}{k}} \cap \{0, \dots, T_{k+1}-1\}$, thus for all $T$ such that $T_k \leq T < T_{k+1}$, we have
\[
\frac{\sum _{\tau \in \Ical \colon \tau \leq T} \gamma_\tau}{\sum _{\tau \in \Nbb \colon \tau \leq T} \gamma_\tau} 
\leq \frac{\sum _{\tau \in \Ical_{\frac{1}{k}} \colon \tau \leq T} \gamma_\tau}{\sum _{\tau \in \Nbb \colon \tau \leq T} \gamma_\tau} 
= \frac{p_k(T)}{\sum _{\tau \in \Nbb \colon \tau \leq T} \gamma_\tau}  \leq \frac{1}{k}
\]
which proves that $\Ical$ has zero density.

Let $\Tcal = \Nbb \setminus \Ical$. We have that $\Tcal$ has density one, and it remains to prove that the subsequence $(u^{(\tau)})_{\tau \in \Tcal}$ converges to $u$. Since $\Tcal$ has density one, it has infinitely many elements, and for all $k$, there exists $S_k \in \Tcal$ such that $S_k \geq T_k$. For all $\tau \in \Tcal$ with $\tau \geq S_k$, there exists $k' \geq k$ such that $T_{k'} \leq \tau < T_{k'+1}$. Since $\tau \notin \Ical$ and $T_{k'} \leq \tau < T_{k'+1}$, we must have $\tau \notin \Ical_{\frac{1}{k'}}$, therefore
\[
\| u^{(\tau)} - u \| < \frac{1}{k'} \leq \frac{1}{k}.
\]
This proves that $(u^{(\tau)})_{\tau \in \Tcal}$ converges to $u$.
\end{proof}

We now present the main result of this section, which concerns the convergence of the sequence of population distributions $(\mu^{(\tau)})$ to the set $\Ncal$ of Nash equilibria. We say that $(\mu^{(\tau)})$ converges to $\Ncal$ if $d(\mu^{(\tau)}, \Ncal) \rightarrow 0$, where $d(\mu, \Ncal) = \inf_{\nu \in \Ncal} \|\mu - \nu\|$.

\begin{theorem} \label{thm:cesaro} Consider a congestion game with discount factors $(\gamma_\tau)_\tau$ satisfying Assumption~\ref{assumption:discounts}. Assume that for all $k \in \{1, \dots, K\}$, population $k$ has sublinear discounted regret. Then the sequence of distributions $(\mu^{(\tau)})_\tau$ converges to the set of Nash equilibria in the sense of \cesaro\ means with respect to $(\gamma_\tau)$. Furthermore, there exists a dense subsequence $(\mu_\tau)_{\tau \in \Tcal}$ which converges to $\Ncal$.
\end{theorem}
\begin{proof} 
First, we observe the following fact:
\begin{lemma}
\label{lem:continuity}
A sequence $(\nu^{(\tau)})$ in $\Delta$ converges to $\Ncal$ only if $(V(\nu^{(\tau)}))$ converges to $V_\Ncal$, the value of $V$ on $\Ncal$.
\end{lemma}

Indeed, suppose by contradiction that $V(\nu^{(\tau)}) \to V_\Ncal$ but $\nu^{(\tau)} \not\to \Ncal$. Then there would exist $\epsilon > 0$ and a subsequence $(\nu^{(\tau)})_{\tau \in \Tcal}$, $\Tcal \subset \Nbb$ such that $d(\nu^{(\tau)}, \Ncal) \geq \epsilon$ for all $\tau \in \Tcal$. Since~$\Delta$ is compact, we can extract a further subsequence $(\nu^{(\tau)})_{\tau \in \Tcal'}$ which converges to some $\nu \notin \Ncal$. But by continuity of $V$, $(V(\nu^{(\tau)}))_{\tau \in \Tcal'}$ converges to $V(\nu) > V_\Ncal$, a contradiction.

Consider the potential function $V$ defined in equation~\eqref{eq:pot_lat}. 
By convexity of $V$ and the expression~\eqref{eq:potential_derivative} of its gradient, we have for all $\tau$ and for all $\mu \in \Delta$:
\begin{equation*}
V(\mu^{(\tau)}) - V(\mu) \leq \braket{\nabla V(\mu^{(\tau)}) }{\mu^{(\tau)} - \mu} = \sum _{k = 1}^K m(\Xcal_k) \braket{\ell^k(\mu^{(\tau)}) }{ {\mu_p^k}^{(\tau)} - \mu^k_p}
\end{equation*}
then taking the time-weighted sum up to iteration $T$,
\begin{align*}
\sum _{\tau = 0}^T \gamma_\tau (V(\mu^{(\tau)}) - V(\mu))
& \leq \sum_{k=1}^K m(\Xcal_k) \left[ \sum_{\tau = 0}^T \gamma_\tau \braket{{\mu^k}^{(\tau)} }{ \ell^k(\mu^{(\tau)}) } - \braket{\mu^k }{ \sum_{\tau = 0}^T \gamma_\tau \ell^k(\mu^{(\tau)}) }\right] \\
&= \sum_{k=1}^K m(\Xcal_k) \left[ {L^k}^{(T)} - \braket{\mu^k }{ {\Lscr^k}^{(T)} }\right] \\
& \leq \sum_{k = 1}^K m(\Xcal_k) {R^k}^{(T)}
\end{align*}
where for the last inequality, we use the fact that $\braket{\mu^k}{{\Lscr^k}^{(T)}} \geq \min_{p \in \Pcal_k} {\Lscr^k_p}^{(T)}$. In particular, when $\mu$ is a Nash equilibrium, by Theorem~\ref{thm:rosenthal}, $V(\mu) = \min_{\mu \in \Delta} V(\mu) = V_\Ncal$, thus
\[
\frac{\sum_{\tau = 0}^T \gamma_\tau |V(\mu^{(\tau)}) - V_\Ncal|}{\sum_{\tau = 0}^T \gamma_\tau} \leq \sum_{k = 1}^K m(\Xcal_k)\frac{{R^k}^{(T)}}{\sum_{\tau = 0}^T \gamma_\tau }
\]
Since the population-wide regret ${R^k}^{(T)}$ is assumed to be sublinear for all $k$, we have $|V(\mu^{(\tau)}) - V_\Ncal| \toCes 0$. By Proposition~\ref{prop:cesaro_implies_dense_subseq}, there exists $\Tcal \subset \Nbb$ of density one, such that $(V(\mu^{(\tau)}))_{\tau \in \Tcal}$ converges to $V_{\Ncal}$. And it follows that $(\mu^{(\tau)})_{\tau \in \Tcal}$ converges to $\Ncal$. This proves the second part of the theorem. To prove the first part, we observe that, by convexity of $V$,
\[
V_{\Ncal} 
\leq V \parenth{ \frac{\sum _{\tau = 0}^T \gamma_\tau \mu^{(\tau)} }{\sum_{\tau = 0}^T \gamma_\tau}  } 
\leq \frac{\sum_{\tau = 0}^T \gamma_\tau V(\mu^{(\tau)}) } {\sum_{\tau = 0}^T \gamma_\tau}
=V_{\Ncal} + \frac{\sum_{\tau = 0}^T \gamma_\tau (V(\mu^{(\tau)}) - V_{\Ncal}) } {\sum_{\tau = 0}^T \gamma_\tau}
\]
and the upper bound converges to $V_\Ncal$. Therefore $\parenth{ \frac{\sum_{\tau \leq T} \gamma_\tau \mu^{(\tau)}}{\sum_{\tau\leq T} \gamma_\tau } }_{T \in \Nbb}$ converges to $\Ncal$.
\end{proof}

To conclude this section, we observe that the \cesaro\ convergence result of Theorem~\ref{thm:cesaro} can be generalized to any game with a convex potential function. 


\section{Continuous-time dynamics}
\label{sec:replicator}
We now turn to the harder question of convergence of $(\mu^{(\tau)})_\tau$: we seek to derive sufficient conditions under which the sequence $(\mu^{(\tau)})$ converges to $\Ncal$. In this section, we study a continuous-time limit of the update equation given by the Hedge algorithm. The resulting ODE, known as the replicator equation, will be useful in proving strong convergence results in the next section.

\subsection{The Replicator dynamics}
\label{subsec:replicator_motivation}
To motivate the study of the replicator dynamics from an online learning point of view, we first derive the continuous-time replicator dynamics as a limit of the discrete Hedge dynamics, as discussed below. Assume that in each population $\Xcal_k$, all players start from the same initial distribution ${\pi^k}^{(0)} \in \Delta^{\Pcal_k}$, and apply the Hedge algorithm with learning rates $(\gamma_\tau)$. As a result, the sequence of distributions $({\mu^k}^{(\tau)})$ satisfies the Hedge update rule~\eqref{eq:hedge_update}. Now suppose the existence of an underlying continuous time $t \in \Rbb^+$, and write $\mucont(t)$ the distribution at time $t$. Suppose that the updates occur at discrete times $T_{\tau}$, $\tau \in \Nbb$, such that the time steps are given by a decreasing, vanishing sequence $\epsilon_\tau$. That is, $T_{\tau+1} - T_\tau = \epsilon_\tau$. Then we have for all $k$ and all $p \in \Pcal_k$, using Landau notation:
\begin{align*}
\mucont^k_p(T_{\tau+1}) &= {\mu^k_p}^{(\tau+1)} \\
&= {\mu_p^k}^{(\tau)} \frac{ e^{-\gamma_\tau \frac{\ell^k_p(\mu^{(\tau)})}{\rho} } }{\sum _{p' \in \Pcal_k} {\mu_{p'}^k}^{(\tau)} e^{-\gamma_\tau \frac{ \ell^k_{p'}(\mu^{(\tau)}) }{ \rho } } } \\
& = {\mu^k_p}^{(\tau)} \frac{ 1- \gamma_\tau \frac{ \ell^k_p( \mu^{(\tau)} ) }{ \rho } + o(\gamma_\tau)  }{1 - \gamma_\tau \sum _{p' \in \Pcal_k} {\mu^k_{p'}}^{(\tau)} \frac{\ell^k_{p'}(\mu^{(\tau)}) }{ \rho } + o(\gamma_\tau)} \\
& = \mucont^k_p(T_\tau) \left[ 1 + \gamma_\tau \frac{ \bar{\ell}^k(\mu^{(\tau)}) - \ell^k_p(\mu^{(\tau)}) }{ \rho }\right] + o(\gamma_\tau)
\end{align*}
Thus,
\begin{equation*}
\frac{ \mucont^k_p(T_{\tau+1}) - \mucont^k_p(T_\tau) } {T_{\tau+1} - T_\tau} \frac{\epsilon_\tau}{\gamma_\tau} 
= \mucont^k_p(T_\tau) \frac{ \bar{\ell}^k(\mu(\tau)) - \ell^k_p(\mu(\tau)) }{ \rho } + o (1)
\end{equation*}

In particular, if we take the discretization time steps $\epsilon_\tau$ to be equal to the sequence of learning rate $\gamma_\tau$, the expression simplifies, and taking the limit as $\gamma_\tau \rightarrow 0$, we obtain the following ODE system:

\begin{equation}
\label{eq:replicator}
\begin{cases}
\mucont(0) \in \mathring{\Simp}
\\ \forall k, \ \forall p \in \Pcal_k, \frac{d \mucont^k_p(t)}{dt} = \mucont^k_p(t) \frac{\bar{\ell}^k(\mucont(t)) - \ell^k_p(\mucont(t)) }{ \rho } 
\end{cases}
\end{equation}
where $ \mathring{\Simp} = \{ \mu \in \Simp \colon \forall k, \ \forall p \in \Pcal_k, \mu^k_p > 0 \}$ is the relative interior of $\Delta$. Here, we require that the initial distribution have positive weights on all bundles for the following reason: whenever $\mucont^k_p(0) = 0$, any solution trajectory will have $\mucont^k_p (t) \equiv 0$. It is impossible for such trajectories to converge to the set of Nash equilibria $\Ncal$ if the support of equilibria in $\Ncal$ contains $p$. In other words, the replicator dynamics cannot expand the support of the initial distribution, therefore we require that the initial distribution be supported everywhere.

Equation~\eqref{eq:replicator} defines a vector field $F : \Delta \rightarrow \Hcal$, where $\Hcal$ is the product $\Hcal = \Hcal^{\Pcal_1} \times \dots \times \Hcal^{\Pcal_K}$, and
\[
\Hcal^{\Pcal_k} = \left\{v \in \Rbb^{\Pcal_k} : \sum_{p \in \Pcal} v_p = 0 \right\}
\]
is the linear hyperplane parallel to the simplex $\Delta^{\Pcal_k}$. Indeed, we have for all $\mu \in \Delta$ and for all $k$,
\al{
\sum_{p \in \Pcal_k} F^k_p(\mu) 
= \sum_{p \in \Pcal_k} \ell^k_p(\mu) \mu^k_p -  \bar{\ell}^k(\mu) \sum_{p \in \Pcal_k} \mu^k_p = 0.
}
The following proposition ensures that the solutions remain in the relative interior and are defined on all times.

\begin{proposition}
\label{prop:ODE_solution}
The ODE~\eqref{eq:replicator} has a unique solution $\mu(t)$ which remains in $\mathring{\Delta}$ and is defined on $\Rbb_+$.
\end{proposition}

\begin{proof}
First, since the congestion functions $c_r$ are assumed to be Lipschitz continuous, so is the vector field $F$. We thus have existence and uniqueness of a solution by the Cauchy-Lipschitz theorem.

To show that the solution remains in the relative interior of $\Delta$, we observe that for all $k$, 
$
\frac{d}{dt} \sum _{p \in \Pcal_k} \mucont^k_p(t) = \sum _{p \in \Pcal_k} F^k_p(\mucont(t)) = 0
$ by the previous observation. Therefore, $\sum_{p \in \Pcal_k} \mucont^k_p(t)$ is constant and equal to $1$. To show that $\mucont^k_p(t) > 0$ for all $t$ in the solution domain, assume by contradiction that there exists $t_0 > 0$ and $p_0 \in \Pcal_k$ such that $\mucont^k_{p_0}(t_0) = 0$. Since the solution trajectories are continuous, we can assume, without loss of generality, that $t_0$ is the infimum of all such times (thus for all $t < t_0$, $\mucont_{p_0}(t) > 0$). Now consider the new system given by
\[
\begin{aligned}
&\dot{\tilde{\mucont}}_p = \frac{1}{\rho}(\bar{\ell} (\tilde{\mucont})  - \ell_p(\tilde{\mucont})) \tilde{\mucont}_p & \forall p \neq p_0 \\
&\tilde{\mucont}_p(t_0) = \mucont_p(t_0) & \forall p \neq p_0
\end{aligned}
\]
and $\tilde{\mucont}_{p_0}(t)$ is identically equal to $0$. Any solution of the new system, defined on $(t_0 - \delta, t_0]$, is also a solution of equation~\eqref{eq:replicator}. Since $\mucont(t_0) = \tilde{\mucont}(t_0)$, we have $\mucont \equiv \tilde{\mucont}$ by uniqueness of the solution. This leads to a contradiction since by assumption, for all $t < t_0$, $\mucont_p(t) > 0$ but $\tilde{\mucont}_p(t) = 0$.

This proves that $\mucont$ remains in $\mathring{\Delta}$. Furthermore, since $\Delta$ is compact, we have by Theorem~2.4 in \cite{khalil1992nonlinear} that the solution is defined on $\Rbb_+$ (otherwise it would eventually leave any compact set).
\end{proof}

Equation~\eqref{eq:replicator} is also studied in Evolutionary Game Theory and is referred to as the \emph{replicator dynamics} (see~\cite{weibull1997evolutionary} for example). It arises from the following model: for all time $t$, players of population $\Xcal_k$ are matched in random pairs, and each pair compares their losses. If the two players have strategies $p, p' \in \Pcal_k$, the player with higher loss imitates the strategy of the other player with probability proportional to the difference in losses (hence the name replicator), that is, if $\ell^k_p(\mucont) > \ell^k_{p'}(\mucont)$, then the first player switches to bundle $p'$ with probability $\frac{\ell^k_p(\mucont) - \ell^k_{p'}(\mucont)}{\rho}$. Under this model, we have
\al{
\frac{d\mucont^k_p}{dt} 
&= \mucont^k_p \parenth{
-\sum_{\substack{p' \in \Pcal_k : \\ \ell^k_p(\mucont) > \ell^k_{p'}(\mucont)}} \mucont^k_{p'} \frac{\ell^k_p(\mucont) - \ell^k_{p'}(\mucont)}{\rho}
+\sum_{\substack{p' \in \Pcal_k : \\ \ell^k_p(\mucont) \leq \ell^k_{p'}(\mucont)}} \mucont^k_{p'} \frac{\ell^k_{p'}(\mucont) - \ell^k_{p}(\mucont)}{\rho}
} \\
&= \mucont^k_p\sum_{p' \in \Pcal_k} \mucont^k_{p'} \frac{\ell^k_{p'}(\mucont) - \ell^k_{p}(\mucont)}{\rho} \\
&= \mucont^k_p\frac{\bar{\ell}^k(\mucont) - \ell^k_{p}(\mucont)}{\rho}
}
which results in the same ODE~\eqref{eq:replicator}.


\subsection{Stationary points of the replicator dynamics}
We first give a characterization of stationary points of the replicator dynamics applied to the congestion game.
\begin{proposition}
A product distribution $\mu$ is a stationary point for the replicator dynamics~\eqref{eq:replicator} if and only if the bundle losses $\ell^k_p(\mu)$ are equal on the support of $\mu^k$.
\end{proposition}

This follows immediately from equation~\eqref{eq:replicator}. We observe in particular that all Nash equilibria are stationary points, but a stationary point may not be a Nash equilibrium in general: one may have a stationary point $\mu$ such that $\mu^k_p = 0$ but $\ell^k_p(\mu)$ is strictly lower than losses of bundles in the support, which violates the condition in Definition~\ref{def:nash} of a Nash equilibrium.

A stationary point $\mu$ with support $\Pcal_1'\times \dots \times \Pcal_K'$ can be viewed as a Nash equilibrium of a modified congestion game, in which the bundle set of each population $\Xcal_k$ is restricted to $\Pcal_k'$. For this reason, stationary points have been called \emph{restricted Nash equilibria} by Fischer and V\"{o}cking in~\cite{fischer2004evolution}. We will denote the set of stationary points by $\Rcal\Ncal$, in reference to the aforementioned paper.

\begin{rem}
\label{rem:finite_res_set}
By the previous observation, a stationary point with support $\Pcal_1'\times \dots \times \Pcal_K'$ is a minimizer of the potential function $V$ on the product $\Delta^{\Pcal_1'} \times \dots \times \Delta^{\Pcal_K'}$. As the number of support sets is finite, the set of potential values of stationary points $V( \Rcal\Ncal)$ is also finite.
\end{rem}

\subsection{Convergence of the replicator dynamics}
\label{subsec:convergence_to_stationary}
In~\cite{fischer2004evolution}, Fischer and V\"{o}cking prove, using a Lyapunov argument, that all solution trajectories of the replicator system asymptotically approach the set of stationary points $\Rcal\Ncal$. Unfortunately, this result only guarantees convergence to a superset of Nash equilibria. However, this will be useful in the next section, and we present a proof for completeness.

\begin{proposition}[Fischer and V\"{o}cling, \cite{fischer2004evolution}]
\label{prop:convergence_to_stationary}
Every solution of the system~\eqref{eq:replicator} converges to the set of stationary points $\Rcal\Ncal$.
\end{proposition}

\begin{proof} Consider the potential function $V$ defined by equation~\eqref{eq:pot_lat}. The function $V$ is continuously differentiable and its derivative along the vector field of the ODE is given by:
\begin{equation*}
\begin{aligned}
\dot{V}(\mu) 
&= \braket{\nabla V(\mu)}{F(\mu)} \\
&= \sum _{k=1} ^K m(\Xcal_k) \sum _{p \in \Pcal_k} \ell^k_p( \mu ) \mu^k_p \left(  \sum_{p' \in \Pcal_k} \mu^k_{p'} \ell^k_{p'}(\mu) -\ell^k_p(\mu) \right) \\
& = \sum _{k=1} ^K m(\Xcal_k) \left[ \left(\sum _{p \in \Pcal_k} \mu^k_p \ell^k_p(\mu) \right)^2 - \sum _{p \in \Pcal_k} \mu^k_p \ell^k_p(\mu)^2 \right]
\end{aligned}
\end{equation*}
By Jensen's inequality, $\dot{V}(\mu) \leq 0$, with equality if and only if $\mu \in \Rcal\Ncal$. Therefore $V$ is defined on the compact set $\Simp$, and is decreasing along the vector field $F$. By the LaSalle-Krasovskii invariance principle, $\mucont(t)$ approaches the largest invariant set contained in the set where $\dot{V}$ vanishes, $\{\mu \in \Delta: \dot{V}(\mu) = 0\} = \Rcal\Ncal$ (for example by Theorem 4.4 in \cite{khalil1992nonlinear}). But since $\Rcal\Ncal$ itself is an invariant set, $\mucont(t)$ approaches $\Rcal\Ncal$. \end{proof}

\subsection{A discrete-time replicator equation: the REP update rule}

Inspired by the continuous-time replicator dynamics, we propose a discrete-time multiplicative update rule by discretizing the ODE~\eqref{eq:replicator}. The resulting algorithm has many desirable properties such as sublinear discounted regret and simplicity of implementation. We call it the REP algorithm in reference to the replicator ODE.

The vector field $F$ can be written in the following form: for all $k$, $F^k(\mucont) = G^k(\mucont, \ell(\mucont))$ where for all $p$,
\[
G^k_p(\mucont, \ell) = \mucont^k_p \frac{\braket{\mucont^k}{\ell^k} - \ell^k_p}{\rho}
\]
This motivates the following update rule for a player $x \in \Xcal_k$ with distribution $\pi^{(\tau)}(x)$:
\[
\pi^{(\tau +1)}(x) = \pi^{(\tau)}(x) + \eta_\tau G^k(\pi^{(\tau)}(x), \ell(\mu^{(\tau)}))
\]

\begin{definition}[Discrete Replicator algorithm]
\label{def:REP}
The REP algorithm, applied by player $x \in \Xcal_k$, with initial distribution $\pi^{(0)} \in \Delta^{\Pcal_k}$ and learning rates $(\eta_\tau)_{\tau \in \Nbb}$ with $\eta_\tau \leq 1$, is an online learning algorithm $(\tensor[^x]{U}{^{(\tau)}})_{\tau \in \Nbb}$ such that the $\tau$-th update function is given by $\tensor[^x]{U}{^{(\tau)}} ( (\ell^k(\mu^{(t)}))_{t \leq \tau}, \pi^{(\tau)} ) = \pi^{(\tau+1)}$, such that
\begin{equation}
\label{eq:rep_update}
\pi^{(\tau+1)}_p - \pi_p^{(\tau)} = \eta_\tau \pi_p^{(\tau)} \frac{ \braket{\pi^{(\tau)}}{\ell^k(\mu^{(\tau)})} - \ell^k_p(\mu^{(\tau)}) }{ \rho }
\end{equation}
\end{definition}

Here, $\braket{\pi^{(\tau)}}{\ell^k(\mu^{(\tau)})} - \ell^k_p(\mu^{(\tau)})$ is the expected instantaneous regret of the player, with respect to bundle $p$. Thus the REP update can also be expressed in terms of the previous distribution and the expected instantaneous regret.

Under the REP update, the sequence of strategy profiles $\pi^{(\tau)}$ remains in the product of simplexes $\Delta$, provided $\eta_\tau \leq 1$ for all $\tau$. Indeed, for all $\tau \in \Nbb$, we have
\begin{equation*}
\sum_{p \in \Pcal_k} \pi _p^{(\tau+1)} 
= \sum_{p \in \Pcal_k} \pi_p^{(\tau)} + \frac{\eta_\tau}{\rho} \left[  \bar{\ell}^k(\mu^{(\tau)}) - \sum_{p \in \Pcal_k} \mu_p^{(\tau)} \ell^k_p(\mu^{(\tau)})  \right] = \sum_{p \in \Pcal_k} \pi_p^{(\tau)}
\end{equation*}
and
\[
1+ \eta_\tau \frac{\bar{\ell}^k(\mu^{(\tau)}) - \ell^k_p(\mu^{(\tau)})  }{ \rho } \geq 1 - \eta_\tau \geq 0
\]
if $\eta_\tau \leq 1$, which guarantees that $\pi^{(\tau)}$ remains in $\Delta$.


We now show that the REP update rule with learning rates $(\gamma_\tau)$ has sublinear discounted regret. First, we prove the following lemma, for general online learning problems with signed losses.

\begin{lemma}
\label{lem:regretbound}
Consider a discounted online learning problem, with sequence of discount factors $(\gamma_\tau)$, with $\gamma_\tau \leq \frac{1}{2}$ for all $\tau$. Let $\Pcal_k$ be the finite decision set, and assume that the losses are signed and bounded, $m_p^{(\tau)} \in [-1, 1]$ for all $\tau$ and $p \in \Pcal$. Then the multiplicative-weights algorithm defined by the update rule
\begin{equation}
\label{eq:mw_signed}
\pi^{(\tau+1)} \propto \parenth{ \pi_p^{(\tau)} (1 - \gamma_\tau m_p^{(\tau)}) }_{p \in \Pcal_k}
\end{equation}
has the following regret bound: for all $\Tau$ and all $p \in \Pcal_k$,
\begin{equation*}
\sum_{0 \leq \tau \leq \Tau} \gamma_\tau \braket{m^{(\tau)}} {\pi^{(\tau)}} \leq  - \log \pi_{\min}^{(0)} + \sum_{0 \leq \tau \leq \Tau} \gamma_\tau m_p^{(\tau)} + \sum_{0 \leq \tau \leq \Tau} \gamma_\tau^2 | m_p^{(\tau)} |
\end{equation*}
where $\pi_{\min}^{(0)} = \min_{p \in \Pcal_k } \pi_p^{(0)}$.
\end{lemma}

\begin{proof}
We extend the proof of Theorem 2.1 in~\cite{arora2012multiplicative} to the discounted case. By a simple induction, we have for all $\Tau$, $\pi^{(\Tau)}$ is proportional to the vector $w^{(\Tau)}$, defined as follows
\[
w^{(\Tau)}_p = \pi_p^{(0)} \prod _{0 \leq \tau < \Tau} (1 - \gamma_\tau m_p^{(\tau)}).
\]
Define the function $\xi^{(\Tau)} = \sum_{p} w_p^{(\Tau)}$. Then $\pi^{(\Tau)}_p = \frac{w^{(\Tau)}_p}{\xi^{(\Tau)}}$, and we have for all $\Tau$:
\begin{align*}
\xi^{(\Tau +1)} 
&= \sum_p w_p^{(\Tau+1)} = \sum_p w_p^{(\Tau)} (1 - \gamma_{\Tau} m_p^{(\Tau)} ) \\
&= \xi^{(\Tau)} - \gamma_{\Tau} \sum_p m_p^{(\Tau)} \pi_p^{(\Tau)} \xi^{(\Tau)} \\
&= \xi^{(\Tau)} \parenth{ 1 - \gamma_{\Tau} \braket{ m^{(\Tau)} }{ \pi^{(\Tau)} } } \\
& \leq \xi^{(\Tau)} e^{ - \gamma_{\Tau} \braket{ m^{(\Tau)} }{ \pi^{(\Tau)} } }
\end{align*}
Thus, by induction on $\Tau$,
\[
\xi^{(\Tau+1)} \leq \exp \left( -\sum_{0 \leq \tau \leq \Tau} \gamma_\tau \braket{ m^{(\tau)} }{ \pi^{(\tau)} } \right)
\]
We also have for all $p$,
\[
\xi^{(\Tau+1)} \geq w_p^{(\Tau+1)} \geq \pi_{\min}^{(0)} \prod_{0 \leq \tau \leq \Tau } (1-\gamma_tm_p^{(\tau)})
\]
Combining the bounds on $\xi^{(\tau)}$ and taking logarithms, we have
\[
\sum_{0 \leq \tau \leq \Tau} \gamma_\tau \braket{m^{(\tau)}}{\pi^{(\tau)}} 
\leq -\log \pi_{\min}^{(0)} - \sum_{0 \leq \tau \leq \Tau} \log(1 - \gamma_\tau m_p^{(\tau)})
\]
To obtain the desired bound, it suffices to show that for all $m \in [-1, 1]$ and $\gamma \in [0, \frac{1}{2}]$,
\[
-\log (1-\gamma m) \leq \gamma m + \gamma^2|m|
\]
Define $h_m(\gamma) = -\log (1-\gamma m) - \gamma m - \gamma^2|m|$. We have $h_m(0) = 0$, and $h_m'(\gamma) = \frac{m}{1-\gamma m} - m - 2\gamma|m| = \frac{\gamma |m|(|m| + 2\gamma m - 2)}{1 - \gamma m}$. Observing that $|m| + 2\gamma m - 2 \leq 0$ for all $m \in [-1, 1]$ and $\gamma \in [0, \frac{1}{2}]$, we have that $h_m'(\gamma) \leq 0$ on $[0, \frac{1}{2}]$, and it follows that $h_m(\gamma) \leq h_m(0) = 0$.
\end{proof}

\begin{proposition} \label{prop:REPregret} If the sequence of discounts $(\gamma_\tau)$ is bounded by $\frac{1}{2}$ and is such that $\sum_{\tau \leq T} \gamma_\tau^2 / \sum_{\tau \leq T} \gamma_\tau$ converges to $0$, then the REP algorithm with learning rates $\gamma_\tau$ has sublinear discounted regret.
\end{proposition}

\begin{proof} Let
\[
r_p^{(\tau)} = \braket{\pi^{(\tau)}}{\ell^k(\mu^{(\tau)} ) } - \ell^k_p(\mu^{(\tau)})  \in [-\rho, \rho]
\]
be the \emph{instantaneous regret} of the player. Then the REP update can be viewed as a multiplicative-weights algorithm with update rule~\eqref{eq:mw_signed}, in which the vector of signed losses is given by $m_p^{(\tau)} = - \frac{{r_p}^{(\tau)} }{ \rho } \in [-1, 1]$, and discount factors $(\gamma_\tau)$. Observing that $\braket{{r}^{(\tau)}}{ \pi^{(\tau)}} = 0$, we have by Lemma~\ref{lem:regretbound}, for all $p \in \Pcal_k$:
\begin{equation*}
\frac{1}{\rho} \sum_{0 \leq \tau \leq T} \gamma_\tau r_p^{(\tau)} \leq  - \log \pi_{\min}^{(0)} + \sum _{0 \leq \tau \leq T} \gamma_\tau^2
\end{equation*}
Rearranging and taking the maximum over $p \in \Pcal_k$, we obtain the following bound on the discounted regret
\begin{align*}
R^{(T)}(x) \leq - \rho\log \pi_{\min}^{(0)} + \rho\sum _{0 \leq \tau \leq T} \gamma_\tau^2
\end{align*}
which shows ${\lim \sup} _{T \to \infty} \frac{1}{\sum_{\tau \leq T} \gamma_\tau} R^{(T)}(x) \leq 0$.
\end{proof}

Interestingly, the REP update can also be obtained as the solution to a regularized version of the greedy update $\min_{\pi \in \Delta^{\Pcal_k}} \braket{\pi}{\frac{\ell^k(\mu^{(\tau)})}{\rho}}$, similarly to the Hedge update (see Proposition~\ref{prop:hedge_reg_div}), with a different regularization function.

\begin{proposition}
\label{prop:rep_reg}
The REP update rule is solution to the following regularized problem:
\begin{equation*}
\{\pi^{(\tau+1)} \} = \arg\min _{\pi \in \Delta} \braket{\pi}{ \frac{ \ell^k(\mu^{(\tau)}) }{ \rho }} + \frac{1}{\eta_\tau} R(\pi \| \pi^{(\tau)} )
\end{equation*}
where $R(\pi \| \nu) = \frac{1}{2} \sum_{p \in \Pcal_k} \pi_p \left( \frac{\pi_p}{\nu_p} -1 \right)^2$ is a divergence measure of $\pi$ from $\nu$.
\end{proposition}

\begin{proof} The proof is similar to Proposition~\ref{prop:hedge_reg_div}. We define the Lagrangian:
\begin{equation*}
\Lcal (\pi ; \lambda ) = \sum _{p \in \Pcal} \pi_p \frac{\ell^k( \mu^{(\tau)} ) }{ \rho } + \frac{1}{2 \gamma_\tau} \sum _{p \in \Pcal_k} \pi_p^{(\tau)} \left( \frac{\pi_p}{\pi_p^{(\tau)}} -1 \right)^2 - \lambda \left( \sum_{p \in \Pcal_k} \pi_p - 1 \right)
\end{equation*}
where $\lambda \in \Rbb$ is the dual variable for the constraint $ \sum_{p \in \Pcal _k} \pi_p = 1$. Its gradient is:
\begin{align*}
\frac{\partial}{\partial \pi_p} \Lcal (\pi ; \lambda) 
&= \frac{\ell^k_p(\mu^{(\tau)}) }{ \rho } + \frac{1}{\gamma_\tau} \left(\frac{\pi_p}{\pi_p^{(\tau)}} - 1 \right) - \lambda \mbox{ forall } p \in \Pcal_k  \\
\frac{\partial}{\partial \lambda} \Lcal (\pi ; \lambda) &= - \sum _{p \in \Pcal_k} \pi_p +1
\end{align*}
and $(\pi^\star, \lambda^\star)$ are primal-dual optimal if and only if
\begin{equation*}
\frac{\pi^\star_p}{\pi_p^{(\tau)}} = 1 + \gamma_\tau \parenth{ \lambda - \frac{ \ell^k_p(\mu^{(\tau)}) }{ \rho } } \mbox{ and } \sum _{p \in \Pcal_k} \pi^\star_p = 1
\end{equation*}
Multiplying by $\pi_p^{(\tau)}$ and taking the sum over $p \in \Pcal_k$, we have $1 = 1 + \gamma_\tau\lambda^\star - \gamma_\tau\braket{ \pi^{(\tau)} }{\frac{\ell^k(\mu^{(\tau)}) }{\rho}}$, i.e. $\lambda^\star = \braket{ \pi^{(\tau)} }{\frac{\ell^k(\mu^{(\tau)}) }{\rho}}$, thus the solution $\pi^\star$ satisfies the REP update rule \eqref{eq:rep_update}.
\end{proof}


\section{Strong convergence of discounted no-regret learning}
\label{sec:conv}

In this section, we give sufficient conditions which guarantee convergence of the sequence of population strategies. The idea is to show that, under these conditions, the discrete process $(\mu^{(\tau)})_{\tau \in \Nbb}$, approaches, in a certain sense, the trajectories of the continuous-time replicator dynamics. Then one can show, using a Lyapunov function, that any limit point of the discrete process must lie in the set of stationary points $\Rcal\Ncal$. With an additional argument, we show that, in fact, limit points lie in the set~$\Ncal$ of Nash equilibria.

We start by reviewing results from the theory of stochastic approximation, which we use in the proof of Theorem~\ref{thm:conv}.

\subsection{Results from the theory of stochastic approximation}
We summarize results from Bena\"{i}m~\cite{benaim1999dynamics}. Let $\Dcal \subset \Rbb^n$, and consider a dynamical system given by the ODE
\begin{equation}
\label{eq:gen_ode}
\dot{\mu} = F(\mu)
\end{equation}
where $F : \Dcal \to \Rbb^n$ is a continuous globally integrable vector field, with unique integral curves which remain in $\Dcal$. Let $\Phi$ be the associated flow function
\begin{align*}
\Phi \colon \Rbb_+ \times \Dcal &\to \Dcal \\
(t, \mu) & \mapsto \Phi_t(\mu)
\end{align*}
such that $t \mapsto \Phi_t(\mu^{(0)})$ is the solution trajectory of~\eqref{eq:gen_ode} with initial condition $\mu(0) = \mu^{(0)}$.

\subsubsection{Discrete-time approximation} We now define what it means for a discrete process to approach the trajectories of the system~\eqref{eq:gen_ode}.

Let $(\mu^{(\tau)})_\tau$ be a discrete-time process with values in $\Dcal$. $(\mu^{(\tau)})_\tau$ is said to be a discrete-time approximation of the dynamical system~\eqref{eq:gen_ode} if there exists a sequence $(\gamma_\tau)_{\tau \in \Nbb}$ of nonnegative real numbers such that $\sum_{\tau \in \Nbb} \gamma_\tau = \infty$ and $\lim_{\tau \to \infty} \gamma_\tau = 0$, and a sequence of deterministic or random perturbations $U^{(\tau)} \in \Rbb^n$ such that for all~$\tau$,
\begin{equation}
\mu^{(\tau + 1)} - \mu^{(\tau)} = \gamma_\tau \parenth{ F(\mu^{(\tau)}) + U^{(\tau+1)} }.
\end{equation}

Given such a discrete-time approximation, we can define the affine interpolated process of $(\mu^{(\tau)})$: let $T_\tau = \sum _{t = 0}^{\tau} \gamma_t$ as in Section~\ref{subsec:replicator_motivation}.

\begin{definition}[Affine interpolated process] The continuous time affine interpolated process of the discrete process $( \mu^{(\tau)} )_{\tau \in \Nbb}$ is the function $M: \Rbb_+ \rightarrow \Rbb^n$ defined by
\begin{equation*}
M(T_\tau + s) = \mu^{(\tau)} + s \frac{ \mu^{(\tau +1)} - \mu^{(\tau)} }{ \gamma_\tau }, \quad \forall \tau \in \Nbb \text{ and } \forall s \in [0, \gamma_\tau ) 
\end{equation*}
\end{definition}

The next proposition gives sufficient conditions for an affine interpolated process to be an APT.

\begin{proposition} [Proposition 4.1 in~\cite{benaim1999dynamics}]
\label{prop:APT}
Let $M$ be the affine interpolated process of the discrete-time approximation $(\mu^{(\tau)})$, and assume that:
\begin{enumerate}[1.]
\item For all $T>0$, 
\begin{equation}
\label{eq:perturbation_condition}
\lim_{\tau_1 \rightarrow \infty} \max_{\tau_2: \sum\limits_{\tau = \tau_1}^{\tau_2} \gamma_\tau < T} \left\| \sum_{\tau = \tau_1}^{\tau_2} \gamma_\tau U^{(\tau + 1)} \right\| = 0
\end{equation}
\item $\sup _{\tau \in \Nbb} \| \mu^{(\tau)} \| < \infty$
\end{enumerate}
Then $M$ is an APT of the flow $\Phi$ induced by the vector field $F$.
\end{proposition}

Furthermore, we have the following sufficient condition for property~\eqref{eq:perturbation_condition} to hold:

\begin{proposition}
\label{prop:robbins}
Let $(\mu^{(\tau)})_{\tau \in \Nbb}$ be a discrete time approximation of the system~\eqref{eq:gen_ode}. Let $(\Omega, \Fcal, \mathbb{P})$ be a probability space and $(\Fcal_\tau)_{\tau \in \Nbb}$ a filtration of $\Fcal$. Suppose that the perturbations satisfy the Robbins-Monroe conditions: for all $\tau \in \Nbb$,
\begin{enumerate}[i)]
\item $U^{(\tau)}$ is measurable with respect to $\Fcal_\tau$
\item $\Exp [U^{(\tau+1)} | \Fcal_\tau ] = 0$
\end{enumerate}

Furthermore, suppose that there exists $q\geq 2$ such that
\begin{align*}
\sup_{\tau \in \Nbb} \Exp [ \| U^{(\tau)} \|^q ] < \infty &&\text{and} && \sum _{\tau \in \Nbb} \gamma_\tau ^{1 + q/2} < \infty
\end{align*}
Then, condition (1) of Proposition~\ref{prop:APT} holds with probability one.
\end{proposition}

\subsubsection{Chain transitivity} We next give an important property of limit points of bounded asymptotic pseudo-trajectories, given in Theorem~\ref{thm:benaim}.
\begin{definition}[Pseudo-orbit and chain transitivity]
A $(\delta, T)$-pesudo-orbit from $a \in \Dcal$ to $b \in \Dcal$ is a finite sequence of partial trajectories. It is given by a sequence of points $(t_i, y_i), i \in \{ 0, \dots, k-1\}$, and the corresponding sequence of partial trajectories
\begin{equation*}
\{ \Phi_t(y_i) \colon 0 \leq t \leq t_i \}; \ i = 0, \ldots, k-1
\end{equation*}
such that $t_i \geq T$ for all $i$, and
\begin{align*}
&d(y_0, a) < \delta \\
&d(\Phi_{t_i} (y_i), y_{i+1}) < \delta, \ i=0, \ldots, k-1 \\
&y_k = b
\end{align*}
The conditions are illustrated in Figure~\ref{fig:orbit}. We write $\Phi \colon a \hookrightarrow _{\delta, T} b$ if there exists a $(\delta, T)$-pesudo-orbit from $a$ to $b$. We write $a \hookrightarrow b$ if $a \hookrightarrow _{\delta, T} b$ for all $\delta, T > 0$. The flow $\Phi$ is said to be chain transitive if $a \hookrightarrow b$ for all $a, b \in \Dcal$.
\end{definition}

\begin{figure}[h]
\centering
\includegraphics[width=0.7\linewidth]{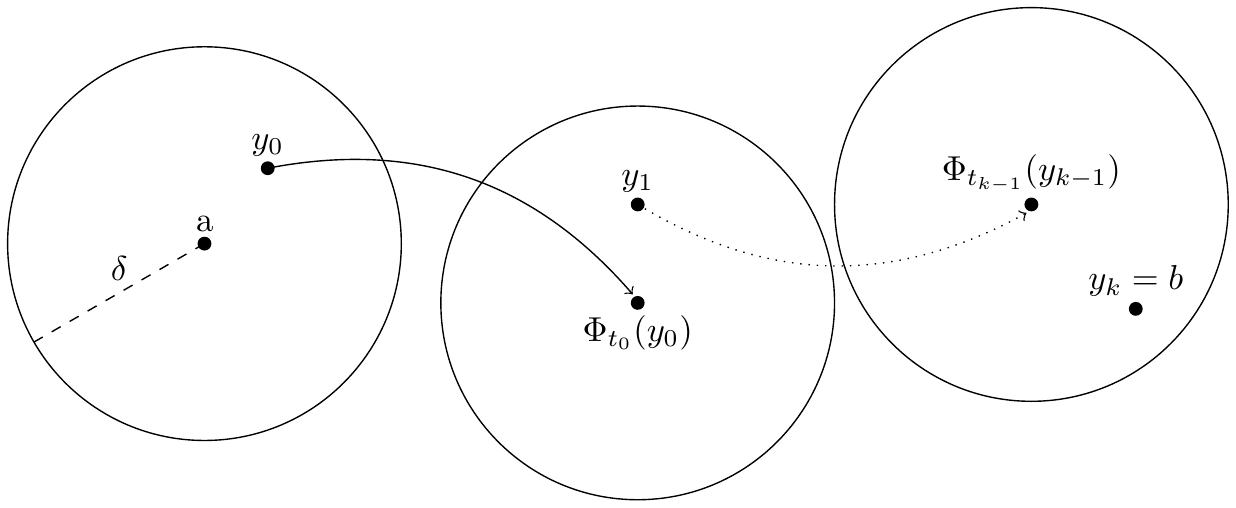} 
\caption{A $(\delta, T)$-pesudo-orbit from $a$ to $b$.}
\label{fig:orbit}
\end{figure}

In the remainder of this section, let $\Gamma \subset \Dcal$ be a compact invariant set for~$\Phi$, that is, $\Phi_t(\Gamma) \subseteq \Gamma$ for all $t \in \Rbb^+$.
\begin{definition}[Internally chain transitive set]
The compact invariant set $\Gamma$ is internally chain transitive if the restriction of $\Phi$ to $\Gamma$ is chain transitive.
\end{definition}

\begin{theorem} [Theorem 5.7 in~\cite{benaim1999dynamics}]
\label{thm:benaim} Let $X$ be a bounded APT of~\eqref{eq:gen_ode}. Then the limit set
\begin{equation*}
L(X) = \bigcap_{t \geq 0} \overline{ \{ X(s) \colon s \geq t \} }
\end{equation*}
is internally chain transitive.
\end{theorem}

Finally, we give the following property of internally chain transitive sets:

\begin{proposition} [Proposition 6.4 in~\cite{benaim1999dynamics}] \label{prop:chaintrans_const} Let $\Gamma \subset \Dcal$ be a compact invariant set and suppose that there exists a Lyapunov function $V: \Dcal \rightarrow \Rbb$ for $\Gamma$ (that is, $V$ is continuous and $\frac{d}{dt}V(x(t)) = \braket{\nabla V(x(t))}{ F(x(t))} < 0$ for all $x \notin \Gamma$) such that $V(\Gamma)$ has empty interior. Then every internally chain transitive set $L$ is contained in $\Gamma$ and $V$ is constant on $L$.
\end{proposition}

\subsection{The AREP class}
Now, we are ready to define a class of online learning algorithms which we call AREP for Approximate REPlicator. An AREP online algorithm can be viewed as a perturbed version of the replicator algorithm.
\begin{definition}[AREP algorithm]
\label{def:AREP}
An online learning algorithm, applied by player $x \in \Xcal_k$, with output sequence $(\pi^{(\tau)})_{\tau \in \Nbb}$, is said to be an approximate replicator (AREP) algorithm if its update equation can be written as
\begin{equation}
\label{eq:AREP}
\pi_p^{(\tau+1)}-\pi_p^{(\tau)} 
= \gamma_\tau \parenth{ \pi_p^{(\tau)} \frac{ \braket{\pi^{(\tau)}}{\ell^k(\mu^{(\tau)})} - \ell^k_p(\mu^{(\tau)}) }{ \rho } + U_p^{(\tau)} }
\end{equation}
where $(U^{(\tau)})_{\tau \in \Nbb}$ is a bounded sequence of stochastic perturbations with values in $\Rbb^{\Pcal_k}$, and which satisfies condition~\eqref{eq:perturbation_condition}.
\end{definition}

In particular, the REP algorithm given in Definition~\ref{def:REP} is an AREP algorithm with zero perturbations. It turns out that the Hedge algorithm also belongs to the AREP class.
\begin{proposition}
\label{prop:HedgeREP}
The Hedge algorithm with learning rates $(\gamma_\tau)_\tau$ satisfying Assumption~\ref{assumption:discounts} is an AREP algorithm.
\end{proposition}

\begin{proof} Let $(\pi^{(\tau)})_{\tau \in \Nbb}$ be the sequence of strategies, and let $(\mu^{(\tau)})_\tau$ be any sequence of population distributions. By definition of the Hedge algorithm, we have
\vspace{-10pt}
\[
\pi_p^{(\tau+1)} = \pi_p^{(\tau)} \exp\parenth{- \gamma_\tau \frac{\ell^k_p(\mu^{(\tau)})}{\rho}} / \sum_{p' \in \Pcal_k} \pi_{p'}^{(\tau)} \exp\parenth{- \gamma_\tau \frac{\ell^k_{p'}(\mu^{(\tau)})}{\rho} }
\]
which we can write in the form of equation~\eqref{eq:AREP}, with perturbation terms
{\small
\al{
U_p^{(\tau+1)} & = \frac{\pi_p^{(\tau)}}{\gamma_\tau} 
\sqbr{ 
\exp \parenth{-\gamma_\tau \frac{ \ell^k_p(\mu^{(\tau)}) - {\tilde{\ell}^k}{}^{(\tau)} }{\rho} } + \gamma_\tau \frac{ \ell^k_p(\mu^{(\tau)}) - \tilde{\ell}^k{}^{(\tau)} }{\rho} -1
}
+ \pi_p^{(\tau)} \frac{\tilde{\ell}^k{}^{(\tau)} - \bar{\ell}^k{}^{(\tau)}}{\rho}
}}%
where
\al{
\tilde{\ell}^k{}^{(\tau)} &= - \frac{\rho}{\gamma_\tau} \log \sum _{p' \in \Pcal_k} \pi_{p'}^{(\tau)} \exp\parenth{- \gamma_\tau \frac{ \ell^k_{p'} (\mu^{(\tau)}) }{ \rho } } \\
\bar{\ell}^k{}^{(\tau)} &= \braket{\pi(\tau)}{\ell^k(\mu(\tau))}
}
Letting $\theta(x) = e^x - x -1$, we have for all $p \in \Pcal_k$:
\begin{align*}
U_p^{(\tau+1)} &= \frac{\pi_p^{(\tau)}}{\gamma_\tau} \theta \left(  - \gamma_\tau \frac{\ell^k_p(\mu^{(\tau)}) - \tilde{\ell}^k{}^{(\tau)} }{\rho} \right) + \frac{\pi_p^{(\tau)}}{\rho} (\tilde{\ell}^k{}^{(\tau)} - \bar{\ell}^k{}^{(\tau)})
\end{align*}
The first term is a $O(\gamma_\tau)$ as $\theta(x) \sim_0 x^2/2$. To bound the second term, we have by concavity of the logarithm
\begin{align*}
\tilde{\ell}^k{}^{(\tau)}
&= - \frac{\rho}{\gamma_\tau} \log \sum _{p' \in \mathcal{P}_k} \pi_{p'}^{(\tau)} \exp\parenth{-\gamma_\tau \frac{\ell_{p'}(\mu^{(\tau)}) }{ \rho } } \\ 
& \leq \sum _{p' \in \mathcal{P}_k} \pi_{p'}^{(\tau)} \ell^k_{p'}(\mu^{(\tau)}) = \bar{\ell}^k{}^{(\tau)}
\end{align*}
And by Hoeffding's lemma,
\begin{align*}
\log \sum _{p' \in \Pcal_k} \pi_{p'} \exp\parenth{-\gamma_\tau \frac{\ell_{p'}(\mu^{(\tau)}) }{ \rho } }   
&\leq - \gamma_\tau \sum _{p' \in \mathcal{P}_k} \pi_{p'}^{(\tau)} \frac{\ell_{p'}(\mu^{(\tau)})}{\rho} + \frac{\gamma_\tau^2}{8}
\end{align*}%
Rearranging, we have $0 \leq \bar{\ell}^k(\tau) - \tilde{\ell}^k(\tau)  \leq  \frac {\rho \gamma_\tau} {8}$, therefore $U_p(\tau +1) = O( \gamma_\tau )$, and
\[
\left\| \sum _{\tau = \tau_1} ^{\tau_2} \gamma_\tau U(\tau+1) \right\|= O \parenth{ \sum _{\tau=\tau_1} ^{\tau_2} \gamma_t^2 }.
\]
Finally, since $\gamma_\tau \downarrow 0$, $\max_{\tau_2 : \sum_{\tau = \tau_1}^{\tau_2}} \sum_{\tau_1}^{\tau_2} \gamma_\tau^2$ converges to $0$ as $\tau_1 \rightarrow \infty$, therefore condition~\eqref{eq:perturbation_condition} is verified.
\end{proof}

\subsection{Convergence of AREP algorithms with sublinear discounted regret}
We now give the main convergence result. 

\begin{theorem}
\label{thm:conv}
Suppose that the population strategies $(\mu^{(\tau)})_\tau$ obey an AREP update rule with sublinear discounted regret. Then $(\mu^{(\tau)})$ converges to the set of Nash equilibria~$\Ncal$.
\end{theorem}

\begin{proof}
By assumption, we have
\al{
\mu^{(\tau+1)}_p - \mu^{(\tau)}_p
&= \gamma_\tau \parenth{ G^k_p\parenth{ \mu^{(\tau)}, \ell(\mu^{(\tau)} ) } + U_p^{(\tau+1)} } \\
&= \gamma_\tau \parenth{ F^k_p(\mu^{(\tau)} ) + U_p^{(\tau+1)} }
}
where, by definition of the AREP class, the perturbations $U^{(\tau)}$ satisfy condition $\emph{1}$ of Proposition~\ref{prop:APT}. Condition $\emph{2}$ is also satisfied since the sequence $(\mu^{(\tau)})_\tau$ lies in the compact set $\Delta$. Thus by Proposition~\ref{prop:APT}, the affine interpolated process $M$ of $(\mu^{(\tau)})_\tau$ is an APT of the continuous-time replicator system $\dot{\mucont} = F(\mucont)$. Thus by Theorem~\ref{thm:benaim}, the limit set $L(M)$ is internally chain transitive.

Consider the set of restricted Nash equilibria $\Rcal\Ncal$. This set is invariant ($\Rcal\Ncal$ is the set of stationary points of the vector field) and compact ($\Rcal\Ncal$ is the finite union of compact sets by Remark~\ref{rem:finite_res_set}). The Rosenthal potential function $V$ is a Lyapunov function for $\Rcal\Ncal$ (see proof of Theorem~\ref{thm:cesaro}), and $V(\Rcal\Ncal)$ has empty interior since it is a finite set by Remark~\ref{rem:finite_res_set}. Therefore we can apply Proposition~\ref{prop:chaintrans_const} to conclude that the set of limit points $L(M)$ is contained in $\Rcal\Ncal$ and $V$ is constant over $L(M)$. Let $v$ be this constant value.

Next, we show that the sequence of potentials $V(\mu^{(\tau)})$ converges. Let $\hat{v}$ be a limit point of $V(\mu^{(\tau)})$. Then by Lemma~\ref{lem:continuity}, $\hat{v} = V(\hat{\mu})$ where $\hat{\mu}$ is a limit point of $(\mu^{(\tau)})$. In particular, $\hat{\mu} \in L(M)$, thus $\hat{v} = V(\hat{\mu}) = v$. This shows that the bounded sequence $(V(\mu^{(\tau)}))$ has a unique limit point $v$, therefore it converges to~$v$, and it remains to show that $v = V_\Ncal$ to conclude (by Lemma~\ref{lem:continuity}).

To show that $v = V_\Ncal$, we first observe that since $V(\mu^{(\tau)}) \to v$, we also have $V(\mu^{(\tau)}) \toCes v$. But the population dynamics is also assumed to have sublinear discounted regret, thus by Theorem~\ref{thm:cesaro}, $V(\mu^{(\tau)}) \toCes V_\Ncal$. By uniqueness of the limit, we must have $v = V_\Ncal$.
\end{proof}

Note that Theorem~\ref{thm:conv} assumes that the AREP update rule is applied to the population dynamics $(\mu^{(\tau)})$, not to individual strategies $\pi^{(\tau)}(x)$. One sufficient condition for $\mu^{(\tau)}$ to satisfy an AREP update is that for each $k$, all players in $\Xcal_k$ start from a common initial distribution ${\pi^k}^{(0)} = {\mu^k}^{(0)}$, and apply the same update rule. This guarantees that for all $\tau$ and for all $x$, $\mu^{(\tau)} = \pi^{(\tau)}(x)$.

\subsection{Convergence of the REP and Hedge algorithms}
We apply Theorem~\ref{thm:conv} to show convergence of the REP and Hedge algorithms.
\begin{corollary}
If $(\mu^{(\tau)})$ obeys the REP update rule with learning rates $\gamma_\tau$ satisfying Assumption~\ref{assumption:discounts} and such that $\gamma_\tau \leq \frac{1}{2}$, then $\mu^{(\tau)} \to \Ncal$.
\end{corollary}
\begin{proof} The REP update rule is a discounted no-regret algorithm by Proposition~\ref{prop:REPregret}, and it is an AREP algorithm with zero perturbations, so we can apply Theorem~\ref{thm:conv}.
\end{proof}

\begin{corollary}
If $(\mu^{(\tau)})$ obeys the discounted Hedge update rule with learning rates $\gamma_\tau$ satisfying Assumption~\ref{assumption:discounts} and such that $\sum_{\tau}\gamma_\tau^2 < \infty$, then $\mu^{(\tau)} \to \Ncal$.
\end{corollary}

\begin{proof} By Proposition~\ref{prop:hedge_bound} and Proposition~\ref{prop:HedgeREP}, the discounted Hedge algorithm with rates $\gamma_\tau$ is an AREP algorithm with sublinear discounted regret, and we can apply Theorem~\ref{thm:conv}.
\end{proof}

\begin{figure}[h!]
\centering
\begin{minipage}{0.5\textwidth}
\centering
\includegraphics[width=\textwidth, page=1]{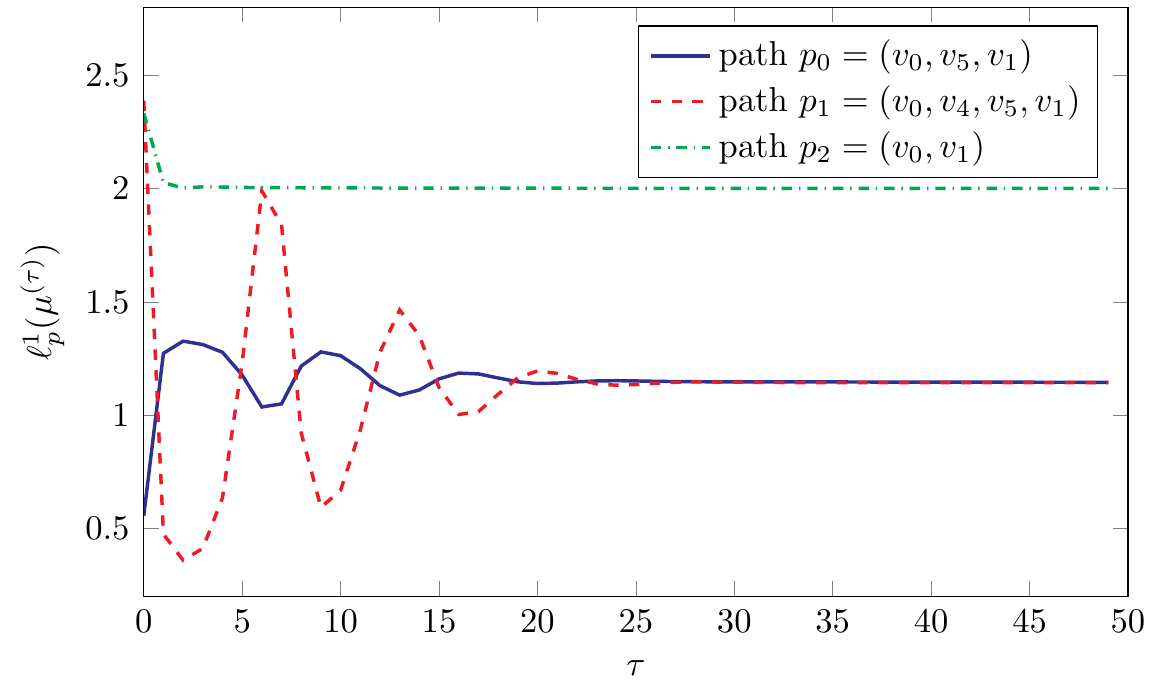}
\includegraphics[width=\textwidth, page=3]{TikZ/hedge-sim}
\includegraphics[height=3.8cm, page=5]{TikZ/hedge-sim}
\end{minipage}%
\begin{minipage}{0.5\textwidth}
\centering
\includegraphics[width=\textwidth, page=2]{TikZ/hedge-sim}
\includegraphics[width=\textwidth, page=4]{TikZ/hedge-sim}
\includegraphics[height=3.8cm, page=6]{TikZ/hedge-sim}
\end{minipage}
\caption{Simulation of the population dynamics under the discounted Hedge algorithm, initialized at the uniform distribution. The trajectories of the population strategies ${\mu^k}^{(\tau)}$ are given in the $2$-simplex for each population (bottom). The path losses $\ell^k_p(\mu^{(\tau)})$ for both populations (top) converge to a common value on the support on the Nash equilibrium. The sequences of discounted regrets (middle) confirm that the population regret is sub-linear, i.e. $\limsup_{t \rightarrow \infty} \frac{{R^k}^{(t)} }{ \sum_{\tau \leq t} \gamma_\tau} \leq 0$.}
\label{fig:sim_hedge}
\end{figure}

We illustrate these corollaries with a routing game on the example network introduced in Section~\ref{subsec:routing_game}. We simulate the population dynamics under the discounted Hedge algorithm (Figure~\ref{fig:sim_hedge}) and the REP algorithm (Figure~\ref{fig:sim_rep}). In both cases, we use a harmonic sequence of learning rates, $\gamma_\tau = \frac{20}{10 + \tau}$. The figures show a typical behavior of the population dynamics: if the starting learning rates are large, the trajectories can first oscillate, but as the learning rates decrease, the trajectories converge to the Nash equilibrium.

\begin{figure}[h]
\centering
\begin{minipage}{0.5\textwidth}
\centering
\includegraphics[width=\textwidth, page=1]{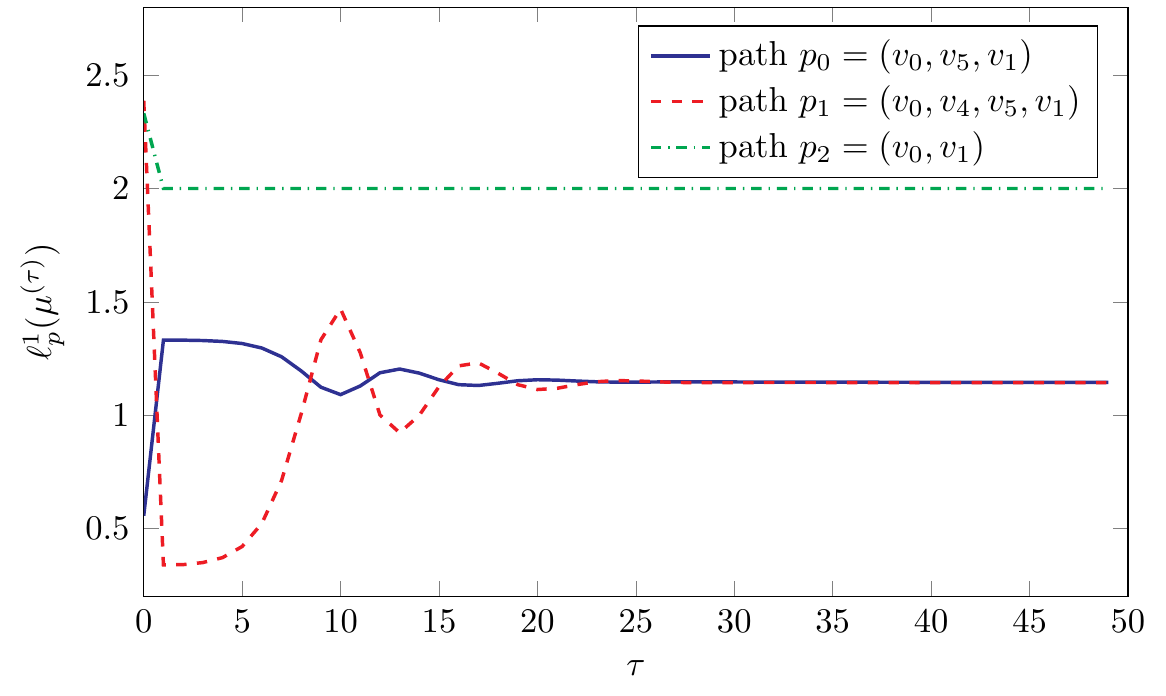}
\includegraphics[width=\textwidth, page=3]{TikZ/rep-sim}
\includegraphics[height=3.8cm, page=5]{TikZ/rep-sim}
\end{minipage}%
\begin{minipage}{0.5\textwidth}
\centering
\includegraphics[width=\textwidth, page=2]{TikZ/rep-sim}
\includegraphics[width=\textwidth, page=4]{TikZ/rep-sim}
\includegraphics[height=3.8cm, page=6]{TikZ/rep-sim}
\end{minipage}
\caption{Simulation of the population dynamics under the REP algorithm.}
\label{fig:sim_rep}
\end{figure}

\section{Conclusion}
\label{sec:conclusion}
We studied the convergence of online learning dynamics in the non-atomic congestion game. We showed that dynamics with sub-linear discounted population regret guarantee the convergence of $(\bar \mu^{(\tau)})$, the sequence of \cesaro\ means of population strategies. To obtain convergence of the actual sequence of strategies $(\mu^{(\tau)})$, we introduced the AREP class of approximate replicator dynamics, inspired by the replicator ODE. We showed that whenever the population strategies obey an AREP dynamics and have sub-linear discounted regret, the sequence converges. These results assume that the sequence of discount factors $(\gamma_\tau)$ is identical for all players. One question is whether this assumption can be relaxed, so that different players can use different learning rates.

\newpage
%
%
%
%
%
%
%
%
%
%
%
%


\bibliographystyle{siam}

\bibliography{learning}

\newpage
\appendix

\section{Proof of Hoeffding's Lemma}
\label{app:hoeffding}
\begin{lemma}[Hoeffding's lemma] \label{lem:hoeffding} Let $X$ be any real-valued random variable such that $X \in [a, b]$ almost surely. Then, for all $s \in \Rbb$:
\begin{equation*}
\log \Exp[e^{sX}] \leq s \Exp[X] + \frac{s^2 (b-a)^2}{8}
\end{equation*}
\end{lemma}

\begin{proof}
Fix $s \in \Rbb$. We first consider the case when $X$ is centered, i.e. $\Exp[X] = 0$. Since $x \mapsto e^{sx}$ is a convex function, we have for all $x \in [a,b]$, 
\begin{equation*}
e^{sx} \leq \frac{b-x}{b-a} e^{sa} + \frac{x-a}{b-a} e^{sb}
\end{equation*}
Thus, 
\begin{equation*}
\Exp[e^{sX}] \leq \frac{b}{b-a} e^{sa} - \frac{a}{b-a} e^{sb} = e^{sa}( 1 - p + p e^{s(b-a)} )
\end{equation*}
where $p = -\frac{a}{b-a}$. Taking the logarithm, we have:
\[
\log (e^{sa}( 1 - p + p e^{s(b-a)} )) = -qp + \log(1-p + pe^q) = g(q)
\]
with $q=s(b-a)$. We have
\al{
g'(q) &= -p + \frac{p e^q }{1-p+pe^q} \\
g''(q) &= \frac{(1-p) pe^q}{(1-p+pe^q)^2}
}
thus $g'(0) = g(0) = 0$, and for all $q$, $g''(q) \leq \frac{1}{4}$. Therefore, by Taylor's expansion, we have $g(q) \leq \frac{q^2}{8}$, and it follows that $\log \Exp[e^{sX}] \leq \frac{s^2(b-a)^2}{ 8 }$.
 
Now consider the general (non-centered) case. We have
\[
\log \Exp [e^{sX}] = \log \left( e^{s\Exp[X]} \Exp [e^{s(X-\Exp[X]}] \right) = s\Exp[X] + \log \Exp [e^{s(X- \Exp[X])}]
\]
and the result follows by applying the centered case to the variable $X - \Exp[X]$.
\end{proof}

\section{Proof of Rosenthal's Theorem}
\label{app:rosenthal_proof}

\begin{proof}
First, write the convex optimization problem as
\begin{equation*}
\begin{aligned}
&\text{minimize}_{\mu_p^k \geq 0, \phi_r} && \sum_{r \in \Rcal} \int_0^{\phi_r} c_r(u) du \\
&\text{subject to} 
&&  \phi = \bar{M} \mu  \\
&&& \forall k, \; \sum _{p \in \Pcal_k} \mu^k_p = 1
\end{aligned}
\end{equation*}
Writing the Lagrangian, we have:
\begin{align*}
L(\mu, &\phi; \lambda, v, w)\\
&= \sum_{r \in \Rcal} \int_0^{\phi_r} c_r(u) du - \sum_{k = 1}^K \sum_{p \in \Pcal_k} \lambda^k_p \mu^k_p + \sum_{r \in \Rcal} v_r((M \mu)_r -\phi_r) - \sum_{k=1}^K w_k \left( \sum_{p \in \Pcal_k} \mu^k_p - 1 \right) 
\end{align*}
where $\lambda^k \in \Rbb_+^{\Pcal_k}$, $v \in \Rbb^\Rcal$ and $w \in \Rbb^K$. Then by strong duality (Slater's condition holds), $(\mu, \phi; \lambda, v, w)$ are primal-dual optimal variables if and only if the following optimality conditions hold (se for example~\cite{boyd2010convex}):
\begin{itemize}
\item Stationarity: for all $r$, $\frac{\partial L}{\partial \phi_r} (\mu, \phi ; \lambda, v, w) = 0$, and $\forall k$, $\forall p \in \Pcal_k$, $\frac{\partial L}{\partial \mu^k_p} (\mu, \phi; \lambda, v, w) = 0$, i.e.
\begin{align*}
&\forall r \in \Rcal, && c_r(\phi_r) - v_r = 0 \\
&\forall k \in \{1, \dots, K\}, \ \forall p \in \Pcal_k, && m(\Xcal_k) \sum_{r \in p} v_r - \lambda^k_p - w_k = 0
\end{align*}
\item Complementary slackness:
\[
\forall k, \ \forall p \in \Pcal_k, \; \lambda^k_p \mu^k_p = 0
\]
\end{itemize}

Therefore we have for all $k$, for all bundles $p \in \Pcal_k$, the bundle loss is
\[
\ell^k_p(\mu) = \sum_{r \in p} c_r(\phi_r) = \sum_{r \in p} v_r = \frac{w_k + \lambda^k_p}{m(\Xcal_k)}
\]
by the stationarity conditions. Additionally, if bundle $p$ has positive mass $\mu^k_p$, by complementary slackness, the corresponding dual variable $\lambda^k_p$ is zero, thus
\[
\mu^k_p > 0 \Rightarrow \ell^k_p(\mu) =  \frac{w_k}{m(\Xcal_k)}
\]
therefore all bundles $p \in \Pcal_k$ with positive mass have a common latency $\frac{w_k}{m(\Xcal_k)}$, and bundles with zero mass have latency greater than or equal to $\frac{w_k}{m(\Xcal_k)}$ (since $\lambda^k_p \geq 0$). This is equivalent to Definition~\ref{def:nash} of a Nash equilibrium of the congestion game.
\end{proof}

Nash equilibria are also said to be \textit{essentially} unique, in the sense of the following proposition.
\begin{proposition}
\label{prop:essential_uniqueness}
The bundle losses $\ell^k_p$ are constant on $\Ncal$.
\end{proposition}
\begin{proof}
Let $\mu_1, \mu_2 \in \mathcal{N}$, and $\lambda \in [0,1]$. We have for all $r \in \Rcal$, $\mu \mapsto \int_0^{(M\mu)_r} c_r(u)du$ is convex. We have
\begin{align*}
V_\Ncal 
&= \lambda V(\mu_1) + (1-\lambda) V(\mu_2)\\
&= \sum_{r \in \Rcal} \left( \lambda \int _0^{(M\mu_1)_r} c_r(\phi) d\phi + (1-\lambda) \int _0 ^{(M\mu_2)_r} c_r(\phi) d\phi \right) \\
&\geq \sum _{r \in \Rcal} \int_0^{(M ( \lambda \mu_1 + (1-\lambda) \mu_2))_r} c_r(u) du & \text{by convexity}\\
&= V(\lambda \mu_1 + (1-\lambda)\mu_2) \\
&= V_\Ncal
\end{align*}
where the last equality follows from the fact that $\lambda \mu_1 + (1-\lambda)\mu_2 \in \Ncal$ since $\Ncal$ is convex. Therefore the convexity inequalities must hold with equality and $\mu \mapsto \int_0^{(M\mu)_r} c_r(u)du$ is linear on $[\mu_1, \mu_2]$. Therefore $\phi \mapsto c_r(\phi)$ is constant on $[(M\mu_1)_r, (M\mu_2)_r]$, and it follows that $\ell^k_p$ is constant on $[\mu_1, \mu_2]$, since $\ell^k_p(\mu) = \sum_{r \in p} c_r((M\mu)_r)$
\end{proof}

\section{Additional proofs}
\begin{fact}
\label{fact:gamma_squares}
Let $(\gamma_\tau)_\tau$ be a positive real sequence satisfying Assumption~\ref{assumption:discounts}, that is, $\gamma_\tau$ is decreasing and $\sum \gamma_\tau = \infty$. Then
\[
\lim_{T\rightarrow \infty} \frac{\sum_{\tau \leq T} \gamma_\tau^2}{\sum_{\tau \leq T} \gamma_\tau} = 0
\]
\end{fact}
\begin{proof}
Fix $\epsilon > 0$. Since $\gamma_\tau$ decreases to $0$, there exists $\tau_1$ such that for all $\tau > \tau_1$, $\gamma_\tau < \epsilon$, thus for all $T > \tau_1$
\al{
\frac{\sum_{\tau = 1}^{T} \gamma_\tau^2}{\sum_{\tau = 1}^{T} \gamma_\tau} 
&\leq \frac{\sum_{\tau \leq \tau_1} \gamma_\tau^2}{\sum_{\tau = 1}^{T} \gamma_\tau}  + \epsilon
}
and the first term converges to $0$, which proves the claim.
\end{proof}

\end{document}